\documentclass{article}
\usepackage{iclr2027_conference,times}

\usepackage[hyphens]{url}
\usepackage[hyperfootnotes=false]{hyperref}
\usepackage{graphicx}
\usepackage{multirow}
\usepackage{array}
\usepackage{enumitem}

\usepackage{amsmath,amssymb}
\usepackage{amsthm}
\usepackage{rotating}
\usepackage{placeins}
\usepackage{capt-of}
\newtheorem{proposition}{Proposition}
\usepackage{booktabs}
\usepackage{cleveref}
\usepackage{makecell}

\title{SPARC: Subspace Position-Aware Robust Few-Shot Calibration for Distribution-Shifted Industrial Anomaly Detection}
\author{
Seokhee Han\thanks{Equal contribution.} \\
Dartmouth College
\And
Seungjun Chu\footnotemark[1] \\
Korea University
\And
Mateusz Nowak \\
Dartmouth College
\And
Peter Chin \\
Dartmouth College
}

\iclrfinalcopy
\begin{document}

\maketitle
\lhead{arXiv preprint}

\begin{abstract}
Vision-based industrial anomaly detectors are calibrated on one distribution but may be deployed on another that differs in illumination, fixture placement, or sensor characteristics, sharply degrading an otherwise accurate detector. Adapting to the incoming lot is a natural response, but labeled anomalies are scarce. We therefore consider calibration using only a handful of verified-normal images available before scoring the rest of the lot. Existing fixes require backpropagation, detector-specific tuning, or choices about feature directions that few calibration samples cannot justify. We present SPARC, a few-shot calibration method that intercepts patch features between encoder and detector and removes a closed-form, spatially indexed estimate of deployment-time nuisance through per-cell subspace projection. It needs only $k \le 8$ verified-normal images and uses the algebraic saturation rank $r{=}k{-}1$ on the encoder's native patch grid. The correction requires no gradient or weight updates and works with memory-bank, density, prototype, and mutual detectors. On the shift-prone benchmarks, SPARC improves pooled Image AUROC and AU-PRO$_{0.3}$ for all seven detectors whose image scores depend on corrected patch features by $+13.8$ and $+3.5$ percentage points (pp), respectively; on benchmarks without engineered shift, the changes are small and mixed. Controls that give competing corrections the same calibration images attribute these gains to the per-cell subspace structure rather than the images alone. Further ablations support the saturation-rank choice and characterize sensitivity to backbone and calibration conditions.
\end{abstract}

\section{Introduction}
\label{sec:intro}

\begin{figure}[t]
\centering
\includegraphics[width=\columnwidth]{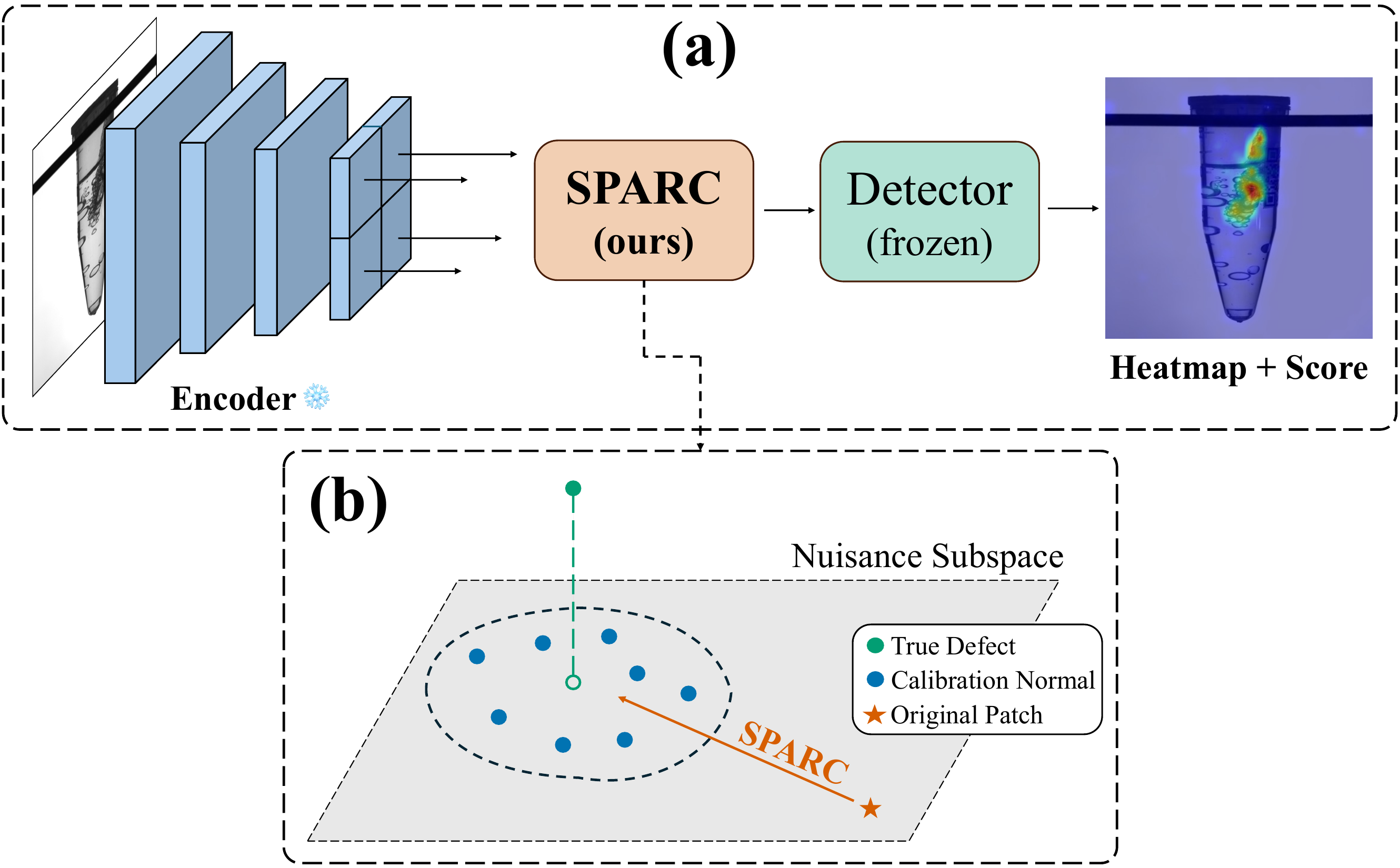}
\caption{\textbf{Overview of SPARC.}
\textbf{(a)}~The pipeline consists of a frozen encoder, SPARC, and a frozen off-the-shelf detector.
\textbf{(b)}~At each cell, SPARC projects out an at-most-$(k{-}1)$-dimensional disagreement subspace estimated from $k{\le}8$ verified normals.}
\label{fig:overview}
\end{figure}

Reliable defect detection is central to industrial quality control as missed defects increase downstream costs and can compromise safety in high-stakes manufacturing.
Modern vision-based inspection commonly combines a frozen pretrained encoder with an anomaly detector fitted on normal images, assuming that deployment and training distributions remain aligned.
In practice, changes in illumination, viewpoint, object pose, scale, background, and imaging conditions often violate this assumption, creating a substantial gap between benchmark and shop-floor performance.
Recent industrial benchmarks confirm that shifts in illumination or viewpoint can substantially degrade otherwise strong anomaly detectors~\citep{zhang2023industrial,hecklerkram2026mvtecad2}.

This deployment shift motivates recalibrating the detector using observations from the incoming lot.
Existing approaches do so in three ways.
\emph{Test-time fine-tuning} of the encoder or detector~\citep{wang2021tent,wang2022cotta,sun2020ttt} requires backpropagation and hyperparameters, and risks overfitting to scarce normal samples.
\emph{Prompt-based methods}~\citep{zhou2024anomalyclip,li2024promptad,jeong2023winclip} require vision-language backbones and introduce prompt-specific design or adaptation.
\emph{Detector-specific calibration}~\citep{nado2020evaluating,schneider2020improving,liao2025robust} adjusts internal statistics or representations for the deployment distribution, but it typically requires detector-specific design choices and offers no principled, data-driven way to determine which feature variation should be removed.
More broadly, a fixed mean correction cannot capture image-dependent nuisance in direction and magnitude, while a more flexible correction may remove discriminative variation. 
SPARC addresses this tradeoff with a rank-limited, detector-agnostic correction fitted in closed form from a few verified-normal images from the incoming lot, without anomaly labels or gradients.
It limits removal to their per-cell disagreement subspace, which it treats as removable deployment nuisance under a per-cell additive-nuisance assumption.
A shared constant offset is absorbed into the per-cell mean and left untouched.

As described in \Cref{fig:overview}, SPARC uses a per-cell SVD of the centered calibration matrix to estimate and remove the within-lot disagreement subspace, whose dimension is at most $k{-}1$.
It uses the algebraic saturation rank $r{=}k{-}1$, operates on the encoder's native patch grid, and applies the projector symmetrically to reference and test features when a training-derived reference is present, requiring no held-out validation.
The result is a detector-agnostic correction module for memory-bank, density, prototype, and mutual detectors.
Reference-based integrations require cached training features and a one-time reference rebuild per deployment lot.
On the shift-prone MVTec AD~2 and AeBAD-S, SPARC improves pooled Image AUROC and AU-PRO$_{0.3}$ for all seven detectors whose image scores depend on corrected patch features.
On VisA and RAD, which contain no engineered shift, the changes are smaller and mixed.

Our contributions are as follows:
\begin{itemize}
    \item We derive SPARC, a lightweight, closed-form per-cell correction that calibrates frozen anomaly detectors from a few verified-normal images without gradient updates or validation-based tuning.

    \item We evaluate four datasets and nine detector configurations, showing broad detection and localization gains under engineered shifts and smaller, mixed changes otherwise.

    \item We compare SPARC with few-shot and feature-level adaptation baselines, while controlled ablations characterize the per-cell design and its sensitivity to backbone and calibration conditions.
\end{itemize}

\section{Related Work}
\label{sec:related}

\subsection{Industrial Anomaly Detection}
Industrial anomaly detection (IAD) identifies rare defective samples from predominantly normal training data~\citep{bergmann2019mvtec,zou2022visa,jezek2021mpdd,wang2024realiad}, usually via reconstruction-, density-, or feature-memory-based approaches.
Reconstruction methods flag anomalies through residuals but can reconstruct defects too well, suppressing the very residuals that reveal them~\citep{zavrtanik2021draem,wyatt2022anoddpm}.
Density methods score low-probability patches under likelihood, flow, or Gaussian models~\citep{rudolph2021differnet,gudovskiy2022cflow,defard2021padim}.
Feature-memory methods such as SPADE and PatchCore compare test embeddings against stored normal features~\citep{cohen2020spade,roth2022patchcore}, increasingly with CLIP- or DINO-based encoders for better transfer and few-shot performance~\citep{radford2021clip,oquab2023dinov2,jeong2023winclip,zhou2024anomalyclip,damm2025anomalydino}.
These detectors perform well off the shelf but assume aligned training and deployment feature distributions~\citep{roth2022patchcore,jeong2023winclip,damm2025anomalydino}, an assumption routinely violated in deployment~\citep{recht2019imagenet,taori2020measuring}.

\subsection{Benchmarks and Deployment Shift}
AeBAD~\citep{zhang2023industrial} evaluates industrial anomaly detection under changes in background, illumination, and viewpoint between training and deployment.
Under its reported evaluation protocol, PatchCore attains $99.1\%$ image-level area under the ROC curve (Image AUROC) on MVTec AD but $71.0\%$ on AeBAD-S.
MVTec AD~2~\citep{hecklerkram2026mvtecad2} introduces advanced industrial scenarios together with unseen test-time lighting conditions.
Its benchmark reports that PatchCore's area under the per-region overlap curve up to an FPR of $0.3$ (AU-PRO$_{0.3}$) decreases from $92.7\%$ on MVTec AD to $53.8\%$ on MVTec AD~2.
Although these cross-benchmark gaps conflate domain shift with differences in categories, defects, image geometry, and overall difficulty, within-benchmark controls confirm substantial degradation under viewpoint shift in AeBAD-S and unseen illumination in MVTec AD~2.
These controlled deployment shifts define the regime SPARC targets.

\subsection{Test-time Adaptation and Calibration Under Distribution Shift}
Distribution shift has been addressed through domain adaptation, test-time training, entropy minimization, normalization-statistic adaptation, and continual adaptation~\citep{sun2020ttt,wang2021tent,nado2020evaluating,schneider2020improving,wang2022cotta,niu2022eata}.
Most of these approaches require ingredients unavailable in few-shot calibration.
Test-time training, entropy minimization, and continual adaptation require updatable parameters, while cross-domain alignment requires separate source and target feature distributions.
These methods therefore cannot be applied unmodified to a frozen detector with only a few verified-normal deployment images.
Normalization-statistic adaptation remains directly applicable, but it models the shift through global or low-order feature statistics rather than a spatially varying patch-level subspace~\citep{nado2020evaluating,schneider2020improving}.
Prompt-based adaptation offers another alternative but is limited to vision-language backbones~\citep{zhou2022coop,zhou2022cocoop,jeong2023winclip,zhou2024anomalyclip}.

CORAL~\citep{sun2016coral} provides a related feature-level approach when source features are available.
It aligns global second-order statistics, whereas SPARC estimates spatially varying low-rank nuisance subspaces.

\subsection{Subspace Methods for Robust Feature Correction}
Subspace projection is well established in dimensionality reduction, domain alignment, and nuisance removal~\citep{eckart1936approximation,jolliffe2002pca,gong2012geodesic,fernando2013subspace,candes2011robust}.
PCA and SVD identify dominant directions, while methods such as geodesic flow kernels and subspace alignment relate low-dimensional source and target subspaces~\citep{gong2012geodesic,fernando2013subspace}.
Such alignment methods generally assume access to sufficiently many samples and require selecting a subspace dimension or alignment structure.
These choices are difficult to validate in few-shot deployment because no defect examples reveal which directions should be preserved~\citep{huang2022regad,xie2023graphcore,fang2023fastrecon}.
SPARC instead estimates a spatially indexed disagreement subspace from the deployment normals alone.
It adopts the algebraic saturation rank $r=k{-}1$, which retains every nonzero disagreement direction without validation-based rank selection.

\subsection{Few-Shot Industrial Anomaly Detection}
Few-shot IAD targets settings in which only a few verified-normal samples are available for a new product or lot~\citep{huang2022regad,xie2023graphcore,jeong2023winclip}.
Most such methods construct or refine a query-conditioned normal reference from those samples.
FastRecon reconstructs a query feature map from support normals and scores the reconstruction residual, and FastRef refines normal prototypes with query features, transferring query characteristics through a learned transform while suppressing anomalies via optimal transport~\citep{fang2023fastrecon,li2026fastref}.
Rather than constructing a query-specific reference, SPARC removes deployment-time nuisance before a frozen detector applies its existing scoring rule.

\section{Method}
\label{sec:method}

Let a pretrained encoder $\Phi$ map an image $x \in \mathbb{R}^{H \times W \times 3}$ to patch features $z = \Phi(x) \in \mathbb{R}^{H' \times W' \times d}$, which an off-the-shelf anomaly detector uses for scoring.
The detector is fitted on a source distribution and deployed to a lot that may differ in illumination, fixture placement, or sensor characteristics.
We assume access to $2 \le k \le 8$ verified-normal calibration images $\mathcal{C} = \{x^{(1)}, \dots, x^{(k)}\}$ from this lot.

SPARC is a gradient-free module inserted between the frozen encoder and detector.
It estimates per-cell disagreement subspaces from $\mathcal{C}$ and removes the corresponding components from patch features.
SPARC uses rank $r{=}k{-}1$ on the encoder's native patch grid and applies the correction symmetrically when the detector uses a training-derived reference.

\subsection{Deployment-time Nuisance Assumption}
\label{sec:method-assumption}
We model the shift between the deployment and training distributions with an additive per-cell nuisance.
Let $f_{i,j}^{(\ell)} \in \mathbb{R}^d$ be the latent feature at spatial cell $(i,j)$ in deployment image $\ell$ (calibration images $\ell = 1, \dots, k$).

Then,
\begin{equation}
\label{eq:nuisance}
f_{i,j}^{(\ell)}
=
f_{i,j}^{\text{train}}
+
\eta_{i,j}^{(\ell)}
+
\epsilon_{i,j}^{(\ell)},
\qquad
\eta_{i,j}^{(\ell)} \in \mathcal{S}_{i,j},
\end{equation}
where $f_{i,j}^{\text{train}}$ is the cell's canonical feature under the training distribution and is constant across $\ell$.
The term $\eta_{i,j}^{(\ell)}$ represents an image-dependent deployment-time nuisance constrained to a cell-specific low-dimensional subspace $\mathcal{S}_{i,j} \subset \mathbb{R}^d$.
The residual $\epsilon_{i,j}^{(\ell)}$ captures remaining feature variation.
Any component shared across all $k$ images merges with $f_{i,j}^{\text{train}}$ into the per-cell mean and is left in place, which is why the Per-cell Subspace Estimator uses centered features.
The centered calibration matrix cannot separate deployment nuisance from intrinsic variation among normal images.
SPARC therefore treats the full subspace spanned by the observed per-cell disagreement as removable deployment nuisance.

Two aspects of \eqref{eq:nuisance} motivate the design of SPARC:

First, the nuisance is modeled as a low-dimensional subspace rather than a fixed offset.
A single offset removes the same vector from every image and cannot capture nuisance that varies in direction or magnitude across images, whereas a shared subspace with image-specific coefficients can.

Second, the model allows the nuisance subspace to differ across spatial cells.
For example, an illumination change may shift edge features differently from interior features, so a single global subspace need not fit all locations.

\subsection{Per-cell Subspace Estimator}
\label{sec:method-estimator}

For each spatial cell $(i,j)$, we stack the $k$ calibration features into
\begin{equation}
X_{i,j}
=
\begin{bmatrix}
f_{i,j}^{(1)} &
f_{i,j}^{(2)} &
\cdots &
f_{i,j}^{(k)}
\end{bmatrix}^{\top}
\in \mathbb{R}^{k \times d}.
\end{equation}
We define the per-cell mean and centered calibration matrix as
\begin{equation}
\bar{\mu}_{i,j}
=
\frac{1}{k}\sum_{\ell=1}^{k} f_{i,j}^{(\ell)},
\qquad
\widetilde{X}_{i,j}
=
X_{i,j}
-
\mathbf{1}_{k}\bar{\mu}_{i,j}^{\top}.
\end{equation}
The $\ell$th row of $\widetilde{X}_{i,j}$ is therefore
$f_{i,j}^{(\ell)}-\bar{\mu}_{i,j}$.

Let $q=\operatorname{rank}(\widetilde{X}_{i,j})$.
Because the centered rows sum to zero,
\begin{equation}
q \leq \min(k{-}1,d).
\end{equation}
When $d\geq k{-}1$, equality holds if the calibration features are in general affine position.
We compute the reduced SVD over the nonzero singular values
\begin{equation}
\label{eq:svd}
\begin{aligned}
\widetilde{X}_{i,j}
&=
U_{i,j}\Sigma_{i,j}W_{i,j}^{\top},\\
W_{i,j}
&=
\begin{bmatrix}
w_{i,j,1} & \cdots & w_{i,j,q}
\end{bmatrix}
\in\mathbb{R}^{d\times q},\\
W_{i,j}^{\top}W_{i,j}
&=
I_q.
\end{aligned}
\end{equation}
For a requested projection rank $r$, let $s=\min(r,q)$ and define
\begin{equation}
V_{i,j}
=
\begin{bmatrix}
w_{i,j,1} & \cdots & w_{i,j,s}
\end{bmatrix}.
\end{equation}
The columns of $V_{i,j}$ span the retained empirical disagreement subspace
$\widehat{\mathcal{S}}_{i,j}^{(r)}=\operatorname{span}(V_{i,j})$, which SPARC treats as removable nuisance under the assumption above.
Consequently, $V_{i,j}V_{i,j}^{\top}$ is the orthogonal projector onto $\widehat{\mathcal{S}}_{i,j}^{(r)}$.

SPARC corrects a test feature by removing the component of its deviation from the calibration mean that lies in this subspace
\begin{equation}
\label{eq:correct}
f_{i,j}^{\mathrm{corr}}
=
f_{i,j}^{\mathrm{test}}
-
V_{i,j}V_{i,j}^{\top}
\bigl(f_{i,j}^{\mathrm{test}}-\bar{\mu}_{i,j}\bigr).
\end{equation}
Equivalently, the corrected feature retains the calibration mean and the component of the test deviation orthogonal to the estimated subspace.
All orthogonal directions remain unchanged, and the host detector applies its original scoring rule to the corrected feature.

\subsubsection{Choosing the Rank}
Since $q\leq k{-}1$, setting $r=k{-}1$ gives $s=q$ and retains every nonzero disagreement direction.
We therefore use
\begin{equation}
\label{eq:rank}
r=k-1.
\end{equation}
This choice requires no held-out validation.
For this default, we write
$\widehat{\mathcal{S}}_{i,j}
=
\widehat{\mathcal{S}}_{i,j}^{(k-1)}$.

\subsubsection{Spatial Grid}
SPARC estimates one subspace for each cell of the encoder's native patch grid $H'\!\times\!W'$.
It introduces no separate spatial grid and requires no grid-size tuning.

\begin{table*}[!t]
\caption{\textbf{Experimental results.} Per benchmark, we report the uncorrected (\emph{base}) and SPARC-corrected (\emph{+S}) Image AUROC and AU-PRO$_{0.3}$ scores at $k{=}8$. All $\Delta$ values are paired SPARC-minus-base differences computed before rounding and reported in percentage points (pp). The pooled columns aggregate category-level $\Delta$ values across the two shift-prone benchmarks (MVTec AD~2 and AeBAD-S) and the two shift-free benchmarks (VisA and RAD), so each category receives equal weight. \textbf{Bold} marks the larger of \emph{base} and \emph{+S} within each dataset cell. AnomalyCLIP$_\mathrm{V}$ and AnomalyCLIP$_\mathrm{M}$ score images from the CLS token, so the per-cell correction leaves their Image AUROC unchanged. It still applies to their patch features and therefore affects AU-PRO.}
\label{tab:main}
\centering
\setlength{\tabcolsep}{1.0pt}
\resizebox{\textwidth}{!}{%
\begin{tabular}{l >{\raggedright\arraybackslash}p{5.5cm} | rrr | rrr | rrr | rrr | rr}
\toprule
& & \multicolumn{3}{c|}{\textbf{MVTec AD~2}} & \multicolumn{3}{c|}{\textbf{AeBAD-S}} & \multicolumn{3}{c|}{\textbf{VisA}} & \multicolumn{3}{c|}{\textbf{RAD}} & \multicolumn{2}{c}{\textbf{Pooled $\Delta$}} \\
\cmidrule(lr){3-5}\cmidrule(lr){6-8}\cmidrule(lr){9-11}\cmidrule(lr){12-14}\cmidrule(lr){15-16}
\textbf{Backbone} & \textbf{Method} & base & +S & $\Delta$ & base & +S & $\Delta$ & base & +S & $\Delta$ & base & +S & $\Delta$ & \shortstack{shift\\prone} & \shortstack{shift\\free} \\
\midrule
\multicolumn{16}{@{}l}{\textit{Image AUROC}} \\
WRN-50 & PatchCore \citep{roth2022patchcore} & 68.7 & \textbf{83.6} & +14.8 & 64.1 & \textbf{66.3} & +2.2 & \textbf{94.3} & 93.9 & $-$0.3 & 92.1 & \textbf{96.2} & +4.0 & +10.6 & +0.8 \\
WRN-50 & PaDiM \citep{defard2021padim} & 59.9 & \textbf{81.6} & +21.7 & 61.6 & \textbf{62.5} & +1.0 & 89.3 & \textbf{89.6} & +0.3 & 97.5 & \textbf{97.8} & +0.3 & +14.8 & +0.3 \\
WRN-50 & SPADE \citep{cohen2020spade} & 53.8 & \textbf{71.0} & +17.2 & 38.3 & \textbf{43.6} & +5.2 & 83.7 & \textbf{84.2} & +0.5 & 28.9 & \textbf{34.7} & +5.8 & +13.2 & +1.8 \\
DINOv2-S/14 & AnomalyDINO \citep{damm2025anomalydino} & 71.7 & \textbf{82.0} & +10.3 & 71.7 & \textbf{73.5} & +1.8 & \textbf{95.3} & 93.7 & $-$1.6 & \textbf{97.2} & 97.0 & $-$0.2 & +7.5 & $-$1.3 \\
DINOv2-G & SubspaceAD \citep{lendering2026subspacead} & 70.9 & \textbf{85.8} & +14.9 & 73.6 & \textbf{88.1} & +14.5 & 98.1 & \textbf{98.2} & +0.1 & 96.3 & \textbf{96.4} & +0.1 & +14.8 & +0.1 \\
CLIP-B/16 & WinCLIP \citep{jeong2023winclip} & 62.9 & \textbf{79.3} & +16.4 & 75.8 & \textbf{76.4} & +0.6 & 75.4 & \textbf{76.6} & +1.1 & 80.5 & \textbf{83.4} & +3.0 & +11.1 & +1.6 \\
CLIP-L/14 & MuSc \citep{li2024musc} & 51.4 & \textbf{84.5} & +33.1 & 30.7 & \textbf{39.0} & +8.2 & \textbf{89.1} & 87.6 & $-$1.4 & \textbf{65.5} & 62.3 & $-$3.2 & +24.8 & $-$1.9 \\
CLIP-L/14 & AnomalyCLIP$_{\mathrm{V}}$ \citep{zhou2024anomalyclip} & 52.5 & 52.5 & 0.0 & 78.9 & 78.9 & 0.0 & 82.1 & 82.1 & 0.0 & 79.6 & 79.6 & 0.0 & 0.0 & 0.0 \\
CLIP-L/14 & AnomalyCLIP$_{\mathrm{M}}$ \citep{zhou2024anomalyclip} & 58.3 & 58.3 & 0.0 & 75.2 & 75.2 & 0.0 & 88.9 & 88.9 & 0.0 & 78.6 & 78.6 & 0.0 & 0.0 & 0.0 \\
\midrule
\multicolumn{16}{@{}l}{\textit{AU-PRO$_{0.3}$}} \\
WRN-50 & PatchCore \citep{roth2022patchcore} & 42.8 & \textbf{45.7} & +2.9 & 85.1 & \textbf{86.3} & +1.2 & \textbf{87.8} & 86.5 & $-$1.3 & 73.2 & \textbf{77.9} & +4.7 & +2.4 & +0.2 \\
WRN-50 & PaDiM \citep{defard2021padim} & 41.0 & \textbf{47.5} & +6.5 & 82.0 & \textbf{82.6} & +0.6 & 82.8 & \textbf{82.9} & +0.1 & 75.5 & \textbf{76.3} & +0.9 & +4.5 & +0.3 \\
WRN-50 & SPADE \citep{cohen2020spade} & 42.0 & \textbf{46.6} & +4.6 & 89.0 & \textbf{89.4} & +0.5 & \textbf{87.9} & 86.0 & $-$1.9 & \textbf{79.9} & 79.3 & $-$0.6 & +3.2 & $-$1.6 \\
DINOv2-S/14 & AnomalyDINO \citep{damm2025anomalydino} & 59.9 & \textbf{63.7} & +3.8 & 83.9 & \textbf{86.7} & +2.8 & \textbf{92.4} & 92.0 & $-$0.4 & 89.4 & \textbf{90.5} & +1.1 & +3.5 & 0.0 \\
DINOv2-G & SubspaceAD \citep{lendering2026subspacead} & 61.3 & \textbf{69.3} & +8.0 & 89.3 & \textbf{92.7} & +3.4 & \textbf{96.1} & 95.9 & $-$0.2 & 94.0 & \textbf{94.6} & +0.6 & +6.5 & 0.0 \\
CLIP-B/16 & WinCLIP \citep{jeong2023winclip} & 45.0 & \textbf{47.1} & +2.1 & 82.6 & 82.6 & 0.0 & 73.6 & \textbf{74.4} & +0.8 & 55.6 & \textbf{56.8} & +1.1 & +1.4 & +0.9 \\
CLIP-L/14 & MuSc \citep{li2024musc} & 53.0 & \textbf{56.7} & +3.7 & 76.9 & \textbf{77.6} & +0.7 & \textbf{91.1} & 88.5 & $-$2.7 & \textbf{90.6} & 87.4 & $-$3.3 & +2.7 & $-$2.8 \\
CLIP-L/14 & AnomalyCLIP$_{\mathrm{V}}$ \citep{zhou2024anomalyclip} & \textbf{45.1} & 45.0 & $-$0.1 & \textbf{88.1} & 86.3 & $-$1.8 & \textbf{83.4} & 83.2 & $-$0.2 & 76.6 & 76.6 & 0.0 & $-$0.7 & $-$0.1 \\
CLIP-L/14 & AnomalyCLIP$_{\mathrm{M}}$ \citep{zhou2024anomalyclip} & 42.6 & \textbf{43.8} & +1.2 & \textbf{89.9} & 89.2 & $-$0.7 & 87.0 & \textbf{87.1} & +0.1 & 67.0 & \textbf{67.6} & +0.6 & +0.6 & +0.2 \\
\bottomrule
\end{tabular}
}%
\end{table*}

\subsection{Plugging SPARC into Existing Detectors}
\label{sec:method-plugin}
SPARC is designed to compose with several classes of existing anomaly detectors that operate on per-patch features, without retraining.
The insertion point depends on how each detector computes its score from $z$.

\subsubsection{(A) Reference-based Detectors}
Reference-based detectors score test patches against normal references constructed from training features~\citep{roth2022patchcore,cohen2020spade,damm2025anomalydino,lendering2026subspacead,defard2021padim,jeong2023winclip}.
To avoid comparing corrected test features with an uncorrected reference, SPARC transforms both the cached training and test features.
The reference is rebuilt once per deployment lot from the transformed cached features using the detector's original procedure.
This requires access to the cached training features but no gradients or weight updates.

\subsubsection{(B) Reference-free (Mutual) Detectors}
MuSc~\citep{li2024musc} does not score against a training-derived reference.
It scores each test patch by mutual comparison against the pool of other test images. 
We therefore correct only the test side and pass each test patch through \eqref{eq:correct} before scoring, so that the mutual comparison operates on corrected features throughout.
With no reference side, symmetric application does not apply.

\subsubsection{(C) CLS-Decoupled Image-Level Detectors}
Some methods such as AnomalyCLIP~\citep{zhou2024anomalyclip} derive the image score directly from the encoder's CLS token and reserve patch features only for the pixel score. 
We therefore apply SPARC only to the patch features, leaving Image AUROC unchanged.
Since the pixel score still uses those features, its AU-PRO can change.

\subsection{Properties}
\label{sec:method-properties}
Key properties follow from SPARC's construction rather than from tuning.
Detailed analysis is provided in \Cref{sec:supp-theory} of the Appendix.

\subsubsection{Algebraic Saturation at \texorpdfstring{$r{=}k{-}1$}{r = k - 1}}
\Cref{eq:rank} sets $r{=}k{-}1$ as a universal saturation rank.
Because $q\leq k{-}1$, any larger requested rank yields the same projector.

Let $s=\dim(\widehat{\mathcal{S}}_{i,j}^{(r)})\leq k{-}1$.
\Cref{eq:correct} therefore modifies only an at-most-$(k{-}1)$-dimensional component and leaves every orthogonal direction unchanged.
Under isotropic centered deviations, the expected removed-energy fraction is $s/d$.
This average-case identity does not guarantee anomaly preservation because real feature directions need not be isotropic.
\Cref{sec:supp-theory,sec:supp-mechanism} of the Appendix provides the derivation and measures anomaly--subspace overlap and anomaly--normal contrast.

\subsubsection{Cost}
Fitting the SPARC projector costs $O(H'W'k^2d)$ and applying it costs $O(H'W'kd)$ per image.
Neither operation uses gradients.
Reference-based detectors additionally incur a one-time, detector-specific cost to transform cached training features and rebuild their references.

\section{Experiments}
\label{sec:experiments}

\subsection{Experimental Setup}
\label{sec:experiments-setup}

\subsubsection{Datasets}
We evaluate on four industrial inspection benchmarks, including two \emph{shift-prone} benchmarks whose test splits differ from their training splits through engineered conditions such as unseen lighting or viewpoint.
These are MVTec AD~2~\citep{hecklerkram2026mvtecad2} with 8 categories and AeBAD-S~\citep{zhang2023industrial} with 4 subdomains, each treated as a category.
We refer to VisA~\citep{zou2022visa} with 12 categories and RAD~\citep{cheng2024rad} with 4 categories as \emph{shift-free} because their splits contain no engineered shift.

\subsubsection{Detectors}
We evaluate SPARC with nine off-the-shelf detector configurations spanning the three integration families.
\begin{enumerate}[label=(\alph*)]
    \item \textbf{Reference-based detectors} include SubspaceAD~\citep{lendering2026subspacead}, AnomalyDINO~\citep{damm2025anomalydino}, PaDiM~\citep{defard2021padim}, PatchCore~\citep{roth2022patchcore}, SPADE~\citep{cohen2020spade}, and WinCLIP~\citep{jeong2023winclip}.
    \item \textbf{Reference-free mutual detectors} are represented by MuSc~\citep{li2024musc}.
    \item \textbf{CLS-decoupled image-level detectors} include AnomalyCLIP~\citep{zhou2024anomalyclip} with prompts pretrained on VisA (AnomalyCLIP$_\mathrm{V}$) and on MVTec AD (AnomalyCLIP$_\mathrm{M}$).
\end{enumerate}
Each detector uses the authors' released implementation, pretrained assets where applicable, and the default scoring procedure, unless otherwise mentioned. 

\subsubsection{Calibration Protocol}
For each dataset--category pair and each calibration size $k \in \{2, 4, 8\}$, we uniformly sample $k$ verified-normal images from the test-split normal pool and use the remainder as the held-out evaluation set.
Both base (without SPARC) and +S (with SPARC) are scored on the same evaluation set, so every reported $\Delta$ is a paired comparison. 
Each experiment uses five random seeds, with the mean score reported.
We sample the calibration images from the test-split normal pool to model verified-normal images from the incoming deployment lot, which is the distribution SPARC is intended to calibrate.
A calibration-source control using training-pool normals is reported in \Cref{sec:supp-mechanism} of the Appendix.

\subsubsection{Metrics}
Following standard practice, we report Image AUROC and AU-PRO$_{\tau}$, the area under the per-region overlap curve up to a false-positive rate limit $\tau$.
AU-PRO is computed with the standard Anomalib implementation at native ground-truth resolution, without any method-specific tuning~\citep{akcay2022anomalib}.
We use $\tau{=}0.3$ on all four benchmarks and additionally report $\tau{=}0.05$ on MVTec AD~2.
The main result tables report scores on a $0$--$100$ scale, with $\Delta$ denoting the SPARC-minus-base difference in pp.
For the two primary metrics, Image AUROC and AU-PRO$_{0.3}$, significance on the shift-prone benchmarks is assessed with a one-sided Wilcoxon signed-rank test (SPARC $>$ base) using seed-averaged per-category $\Delta$s as paired samples.
Holm correction is applied jointly across the seven detectors and two metrics, with full results in \Cref{sec:supp-statistics} of the Appendix.

\subsubsection{Implementation}
SPARC is implemented in PyTorch.
Experimental runs were distributed across three compute nodes with 16 GPUs in total.
Each run was executed independently on a single GPU.
Full hardware and software details are provided in \Cref{sec:supp-impl} of the Appendix.
Code and configurations will be released upon acceptance.

\subsection{Quantitative and Qualitative Results}
\label{sec:experiments-main}
\begin{table}[!t]
\caption{\textbf{MVTec AD 2 localization results at low false-positive rates.} AU-PRO$_{0.05}$ at $k{=}8$. $\Delta$ is the paired SPARC-minus-base difference computed before rounding and reported in percentage points (pp).}
\label{tab:main-mvtec-aupro005}
\centering
\setlength{\tabcolsep}{4pt}
\begin{tabular}{lrrr}
\toprule
\textbf{Method} & base & +S & $\Delta$ \\
\midrule
PatchCore & 19.0 & \textbf{20.4} & +1.3 \\
PaDiM & 11.3 & \textbf{17.1} & +5.8 \\
SPADE & 18.4 & \textbf{19.5} & +1.1 \\
AnomalyDINO & 37.0 & \textbf{38.8} & +1.8 \\
SubspaceAD & 35.4 & \textbf{45.6} & +10.2 \\
WinCLIP & 17.1 & \textbf{17.7} & +0.7 \\
MuSc & 26.6 & \textbf{29.9} & +3.2 \\
AnomalyCLIP$_{\mathrm{V}}$ & 25.2 & \textbf{25.5} & +0.3 \\
AnomalyCLIP$_{\mathrm{M}}$ & 25.2 & \textbf{25.5} & +0.4 \\
\bottomrule
\end{tabular}
\end{table}

\Cref{tab:main} reports Image AUROC and AU-PRO$_{0.3}$ results for every detector, with and without SPARC, under a strictly paired protocol.
\Cref{tab:main-mvtec-aupro005} additionally reports AU-PRO$_{0.05}$ on MVTec AD~2.
The host detector, its source training, the encoder, and the held-out evaluation set are all held fixed, and only the calibration correction is toggled under the setup described above.

Across the seven detectors whose image scores depend on corrected patch features, SPARC improves the shift-prone pooled results by an average of $+13.8$ percentage point (pp) in Image AUROC and $+3.5$~pp in AU-PRO$_{0.3}$.
After Holm correction, the Image AUROC gains are significant for all seven detectors, while five show significant AU-PRO$_{0.3}$ gains.
All seven also improve in the detector-level mean on each reported MVTec AD~2 metric (\Cref{tab:main,tab:main-mvtec-aupro005}).
On shift-free VisA and RAD, the changes are smaller and mixed across detectors.

Because the baseline does not use the calibration images, the base-versus-SPARC gap also reflects the additional deployment information available to SPARC.
The matched-budget comparisons in \Cref{sec:comparison} isolate the effect of the correction design by giving each method the same $k$ normal images.

Representative qualitative results are shown in \Cref{fig:qual-main}, with additional examples in \Cref{sec:supp-qualitative} of the Appendix.

\FloatBarrier

\begin{figure*}[!t]
\centering
\includegraphics[width=\textwidth]{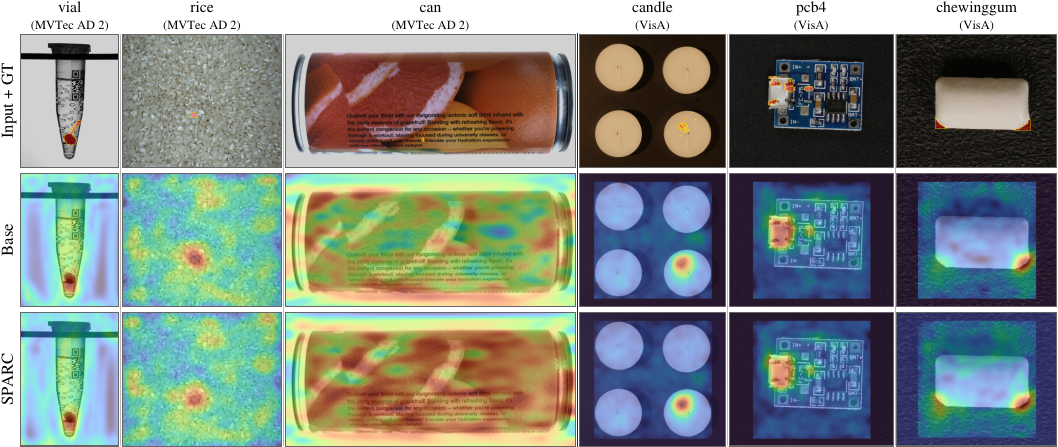}
\caption{\textbf{Qualitative results with PatchCore.} Base and SPARC share a common color scale within each column. Results use deployment-lot calibration with $k{=}8$ and seed$=0$. Score maps are restored to native image coordinates. From left to right, the displayed categories are vial, rice, and can from MVTec AD 2, followed by candle, PCB4, and chewing gum from VisA.}
\label{fig:qual-main}
\end{figure*}

\subsection{Comparison}
\label{sec:comparison}

We compare SPARC with query-conditioned few-shot baselines and matched-budget CORAL and batch-normalization (BN) controls that use the same seed-specific $k{=}8$ calibration draws and evaluation subsets.
We also include transductive full-normal diagnostics estimated from all ground-truth normal deployment images, including normals in the evaluation subset.
These nondeployable information-budget controls test whether few-shot statistic estimation limits the matched-budget variants.
\Cref{tab:normalization-comparison} reports the feature-level comparison on the shared WRN-50 detectors PaDiM, PatchCore, and SPADE.

\begin{table}[!t]
\caption{\textbf{Feature-level adaptation controls.} All entries are paired changes in percentage points (pp). Values are macro-averaged over categories and methods within each seed and reported as mean$\pm$sample SD across five seeds. $^\dagger$Full-normal diagnostics use all ground-truth normal deployment images, including normals in the evaluation subset, and are nondeployable transductive controls.}
\label{tab:normalization-comparison}
\centering
\setlength{\tabcolsep}{2.2pt}
\renewcommand{\arraystretch}{1.02}
\begin{tabular}{@{}lrrr@{}}
\toprule
Variant
& $\Delta$ Image AUROC
& $\Delta$ AU-PRO$_{0.3}$
& \shortstack{MVTec AD~2\\$\Delta$ AU-PRO$_{0.05}$} \\
\midrule

\multicolumn{4}{@{}l}{\textit{Shift-prone benchmarks (MVTec AD~2 and AeBAD-S)}} \\
CORAL few-shot & +1.8$\pm$0.4 & $-$2.8$\pm$0.1 & $-$2.5$\pm$0.1 \\
Calibration-only BN & $-$0.5$\pm$0.4 & +2.1$\pm$0.2 & +4.4$\pm$0.1 \\
Source-mixed BN & +1.5$\pm$0.2 & +4.0$\pm$0.1 & +6.1$\pm$0.2 \\
CORAL full-normal$^\dagger$ & +3.0$\pm$0.4 & $-$2.7$\pm$0.0 & $-$2.6$\pm$0.0 \\
BN full-normal$^\dagger$ & +0.4$\pm$0.1 & +2.3$\pm$0.0 & +4.7$\pm$0.0 \\
SPARC (ours) & +12.9$\pm$1.4 & +3.4$\pm$0.8 & +2.7$\pm$1.6 \\

\midrule
\multicolumn{4}{@{}l}{\textit{Shift-free benchmarks (VisA and RAD)}} \\
CORAL few-shot & +2.5$\pm$0.5 & $-$3.8$\pm$0.1 & -- \\
Calibration-only BN & +1.4$\pm$0.6 & $-$1.3$\pm$0.5 & -- \\
Source-mixed BN & +0.2$\pm$0.1 & +1.7$\pm$0.0 & -- \\
CORAL full-normal$^\dagger$ & +3.3$\pm$0.2 & $-$3.8$\pm$0.0 & -- \\
BN full-normal$^\dagger$ & +2.0$\pm$0.1 & $-$1.1$\pm$0.0 & -- \\
SPARC (ours) & +1.0$\pm$0.2 & $-$0.4$\pm$0.1 & -- \\
\bottomrule
\end{tabular}

\end{table}

Detailed results are provided in \Cref{sec:supp-baselines} of the Appendix.

\subsubsection{Query-Conditioned Few-Shot Detectors}
\label{sec:comparison-few-shot}
Compared with FastRecon and FastRef, the strongest SPARC-equipped detectors match or exceed both baselines across benchmarks, with the clearest margin on the shift-prone AeBAD-S, while weaker hosts fall behind.
Because the reference-based SPARC hosts also retain their training-derived normal references, these absolute scores do not constitute an information-matched comparison.

\subsubsection{Global Feature Alignment}
\label{sec:comparison-coral}
CORAL aligns global second-order statistics, and its matched-budget and full-normal conditions produce nearly identical per-detector changes while often degrading localization, with the largest drops on PaDiM and PatchCore.
This stability across calibration budgets indicates that few-shot estimation alone does not explain CORAL's behavior.
For most detectors, its detection gains are small, and SPARC's shift-prone Image AUROC gain is substantially larger, though CORAL raises SPADE sharply from a low base.

\subsubsection{Batch-Normalization Statistic Adaptation}
\label{sec:comparison-bn}
We evaluate calibration-only and source-mixed BN adaptation using the same $k{=}8$ images.
The full-normal diagnostic does not remove the detector-dependent tradeoff.
SPARC provides a substantially larger shift-prone Image AUROC gain than the matched-budget BN controls, while source-mixed BN performs better on localization.

\subsection{Ablation Studies}
\label{sec:analysis}
Full results for all three ablations are provided in \Cref{sec:supp-design} of the Appendix.

\subsubsection{Correction Mode}
\label{sec:design-correction-mode}
We compare per-cell SPARC with mean-only and global-subspace corrections using the same host detectors and the same $k{=}8$ calibration draws.
The proposed per-cell SPARC achieves the best pooled result on all three metrics, while the global correction remains near the baseline and the mean-only correction can be harmful.
These information-matched controls support the importance of the per-cell subspace design.
Further analysis of the correction mechanism is provided in \Cref{sec:supp-mechanism} of the Appendix.

\subsubsection{Projection Rank}
\label{sec:design-rank}
\begin{figure*}[!t]
\centering
\includegraphics[width=\textwidth]{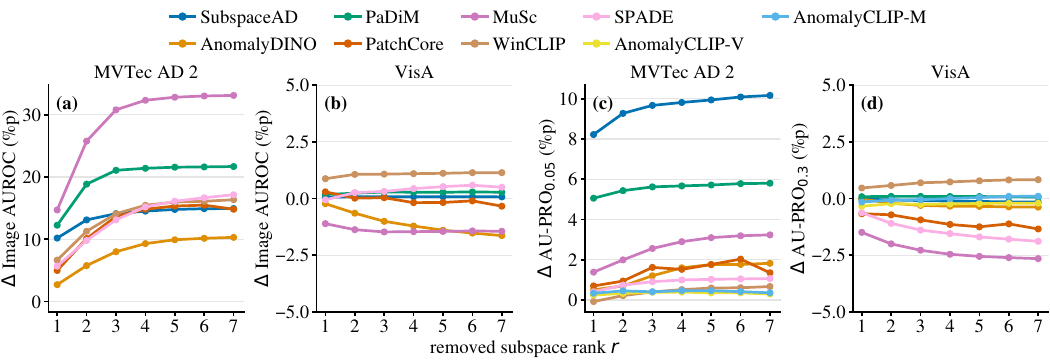}
\caption{\textbf{Projection rank ablation.} Panels (a) and (b) show $\Delta$ Image AUROC on shift-prone MVTec~AD~2 and shift-free VisA. Panel (c) shows $\Delta$ AU-PRO$_{0.05}$ on MVTec~AD~2, and panel (d) shows $\Delta$ AU-PRO$_{0.3}$ on VisA. AnomalyCLIP appears only in the AU-PRO panels.}
\label{fig:ab2-rank}
\end{figure*}

Mean centering limits the empirical disagreement subspace to at most $k{-}1$ dimensions, making $r{=}k{-}1$ the saturation rank.
As shown in \Cref{fig:ab2-rank}, the shift-prone gains generally increase toward this value, although the exact empirical optimum varies by detector and metric.
The changes on shift-free VisA remain comparatively small and mixed.
These results support $r{=}k{-}1$ as a default fixed by the calibration size rather than tuned per benchmark.

\subsubsection{Symmetric Application}
\label{sec:design-symmetric}
For reference-based detectors, one-sided correction compares projected features with an unprojected reference.
On PatchCore, the proposed symmetric correction is more reliable than either one-sided variant and performs best on both shift-prone benchmarks. 

\subsection{Robustness and Sensitivity}
\label{sec:robustness}

We examine sensitivity to calibration budget, imperfect normal verification, and encoder choice, with detailed results in \Cref{sec:supp-robustness} of the Appendix.

\subsubsection{Calibration Size}
Across $k\in\{2,4,8\}$, gains generally increase on the shift-prone benchmarks, while changes on the shift-free benchmarks stay small and lack a consistent trend with $k$.

\subsubsection{Calibration Contamination}
Replacing up to four of the eight calibration normals with same-category anomalies leaves the pooled Image AUROC gain positive through three contaminated samples, while localization degrades earlier.

\subsubsection{Backbone}
Across eight PatchCore backbones, Image AUROC and AU-PRO$_{0.3}$ improve on MVTec AD~2 for every backbone and on AeBAD-S on average, although the AeBAD-S gains vary in sign across backbones; changes on the shift-free benchmarks remain small on average.

\subsection{Cross-Scene Transfer}
\label{sec:cross-scene-transfer}
The main protocol evaluates direct calibration to the incoming lot, while a complementary scene-disjoint split removes all captures associated with calibration-selected scenes to test transfer to unseen physical scenes within the same shifted domain.
Detection gains are smaller but remain positive for most affected detectors, while localization results are mixed, with details in \Cref{sec:supp-scene-disjoint} of the Appendix.

\section{Conclusion}
\label{sec:conclusion}
We presented SPARC, a lightweight correction that estimates and removes a per-cell disagreement subspace from $k{\le}8$ verified-normal deployment images.
Its saturation rank $r{=}k{-}1$, native patch grid, and symmetric application require no held-out validation.
On the shift-prone MVTec AD~2 and AeBAD-S, SPARC produces positive pooled detection and localization changes for all seven detectors with patch-feature image scores.
Changes on the shift-free VisA and RAD are comparatively small and mixed.
SPARC requires spatially aligned patch features and can be hurt when within-lot disagreement reflects intrinsic normal variation or defect signal rather than removable deployment variation.

\bibliography{references}
\bibliographystyle{iclr2027_conference}

\renewcommand{\theHequation}{appendix.\arabic{equation}}
\renewcommand{\theHtable}{appendix.\arabic{table}}
\renewcommand{\theHfigure}{appendix.\arabic{figure}}
\appendix 
\section{Theoretical Analysis}
\label{sec:supp-theory}

This section provides formal statements for the algebraic saturation of the projection rank, the expected removed energy under isotropy, and the computational cost of SPARC.

\subsection{Algebraic Saturation at \texorpdfstring{$r = k - 1$}{r = k - 1}}

\begin{proposition}[Saturation of the projection rank]
\label{prop:rank}
Let $X_{i,j}\in\mathbb{R}^{k\times d}$ be the calibration sample matrix at cell $(i,j)$ and let
\begin{equation}
\widetilde X_{i,j}
=
X_{i,j}-\mathbf{1}_k\bar\mu_{i,j}^{\top}.
\end{equation}
Set
$q=\operatorname{rank}(\widetilde X_{i,j})\leq\min(k{-}1,d)$
and fix an ordered set of right singular vectors
$w_{i,j,1},\ldots,w_{i,j,q}$ associated with the nonzero singular values.
For a requested rank $r$, define
\begin{equation}
\widehat{\mathcal{S}}_{i,j}^{(r)}
=
\operatorname{span}
\bigl(
w_{i,j,1},\ldots,w_{i,j,\min(r,q)}
\bigr).
\end{equation}
Then the subspaces $\widehat{\mathcal{S}}_{i,j}^{(r)}$ are nested in $r$ and satisfy
\begin{equation}
\widehat{\mathcal{S}}_{i,j}^{(r)}
=
\operatorname{rowspace}(\widetilde X_{i,j})
\qquad
\text{for every } r\geq q.
\end{equation}
Since $q\leq k{-}1$, the requested rank $r=k{-}1$ always recovers the full disagreement subspace.
If $d\geq k{-}1$ and the calibration features are in general affine position, then $q=k{-}1$.
\end{proposition}

\begin{proof}[Proof sketch]
Let
\begin{equation}
H_k
=
I_k-\frac{1}{k}\mathbf{1}_k\mathbf{1}_k^{\top}.
\end{equation}
Then $\widetilde X_{i,j}=H_kX_{i,j}$ and
$\operatorname{rank}(H_k)=k{-}1$.
Therefore,
\begin{equation}
q
=
\operatorname{rank}(\widetilde X_{i,j})
\leq
\min(k{-}1,d).
\end{equation}
For the fixed ordered SVD basis, increasing $r$ adds right singular vectors until all $q$ nonzero directions are retained.
Their span is exactly the row space of $\widetilde X_{i,j}$, so the subspace and its projector are unchanged for every $r\geq q$.
If the calibration features are in general affine position and $d\geq k{-}1$, centering introduces the only linear dependence and gives $q=k{-}1$.
\end{proof}

\subsection{Removed Energy Under Isotropy}

\begin{proposition}[Isotropic removed-energy identity]
\label{prop:capacity}
Let $V_{i,j}\in\mathbb{R}^{d\times s}$ have orthonormal columns, where $s\leq k{-}1$, and let
\begin{equation}
\label{eq:supp-removed}
\delta_{i,j}(f)
:=
V_{i,j}V_{i,j}^{\top}(f-\bar\mu_{i,j})
\end{equation} 
be the component removed from a feature $f$.
Let $g=f-\bar\mu_{i,j}\neq 0$ and assume that
$u=g/\|g\|_2$ is uniformly distributed on the unit sphere in $\mathbb{R}^d$.
Then
\begin{equation}
\label{eq:supp-capacity}
\mathbb{E}
\left[
\frac{\|\delta_{i,j}(f)\|_2^2}{\|g\|_2^2}
\right]
=
\frac{s}{d}
\leq
\frac{k-1}{d},
\end{equation}
and deterministically
\begin{equation}
0
\leq
\frac{\|\delta_{i,j}(f)\|_2^2}{\|g\|_2^2}
\leq
1.
\end{equation}
The upper bound is attained if and only if $g\in\operatorname{span}(V_{i,j})$.
\end{proposition}

\begin{proof}[Proof sketch]
Let $P_{i,j}=V_{i,j}V_{i,j}^{\top}$.
Since $P_{i,j}$ is an orthogonal projector,
$\operatorname{tr}(P_{i,j})=s$ and
$\mathbb{E}[uu^{\top}]=I_d/d$.
Therefore,
\begin{equation}
\mathbb{E}\!\left[\|P_{i,j}u\|_2^2\right]
=
\operatorname{tr}
\left(
P_{i,j}\mathbb{E}[uu^{\top}]
\right)
=
\frac{s}{d}.
\end{equation}
The deterministic bounds and equality condition follow from the contractivity of orthogonal projection.
\end{proof}

\Cref{prop:capacity} quantifies the expected fraction of feature energy removed only under an isotropic directional model.
It does not imply neutral detector performance or guarantee preservation of anomaly features because real feature directions may be anisotropic and detector scores are not determined by feature energy alone.
The anomaly-subspace and anomaly-normal contrast diagnostics in \Cref{tab:anomaly-direction-diagnostics} examine these effects empirically.

\subsection{Computational Cost}

For an encoder producing $H'\!\times\!W'$ patch features of dimension
$d$, fitting SPARC at calibration size $k$ requires:
\begin{itemize}
\setlength\itemsep{1pt}
  \item Centering the calibration matrix at each cell:
        $O(H'W' \cdot kd)$.
  \item Compact SVDs of the $k \times d$ centered matrices, one per cell:
        $O(H'W' \cdot k^2 d)$.
\end{itemize}
The compact-SVD cost dominates.
For $H'{=}W'{=}14$, $k{=}8$, and $d{=}768$, the leading-order term is on the order of $10^7$ arithmetic operations.
At inference, applying SPARC requires $O(ds)\leq O(dk)$ operations per cell because the projector is applied through a basis $V_{i,j}\in\mathbb{R}^{d\times s}$ and is never materialized as a $d\times d$ matrix.
The total inference cost is therefore $O(H'W'kd)$ per image.
These bounds cover fitting and applying the SPARC projector.
Reference-based detectors incur an additional one-time cost to transform cached training features and rebuild their references.
This detector-specific cost is not included in the bounds above.

\section{Implementation Details}
\label{sec:supp-impl}

\subsection{Calibration Matrix Construction}
Given $k$ verified-normal calibration images, we run each through the frozen encoder $\Phi$ once and cache the resulting patch-feature maps. 
The calibration matrix $X_{i,j}$ at cell $(i,j)$ is then formed by stacking the $k$ patch features at that cell. 
We do not unroll the spatial grid or otherwise mix features across cells: each of the $H'W'$ matrices is fitted independently. 
The retained basis $V_{i,j}\in\mathbb{R}^{d\times s}$ is stored along with the mean $\bar\mu_{i,j}$ as the entire deployment-time state of SPARC.

\subsection{Symmetric Application for Reference-based Detectors}
We apply the correction to the cached training patch features at calibration time before rebuilding the detector-specific reference (e.g., coreset and KNN index for memory-bank methods or per-position Gaussian statistics for PaDiM).
The same correction is applied to every test patch before scoring. Thus, the train-derived reference and test features lie in the same projected space, while the detector's fitting and scoring algorithms remain unchanged.

\subsection{Random Seeds}
All reported numbers in the main body and in this appendix are averaged over five deterministic seeds for the calibration draw.
The seeds control only which $k$ verified-normal images are sampled from the test-split normal pool; everything downstream is deterministic given the seed.

\subsection{Reference Implementations}
All detectors are implemented based on their official GitHub repositories. 
As FastRef does not have an official reference implementation at the time of our experiments, we reimplement the Euclidean prototype-refinement variant described in Algorithm 1 and apply it to WRN-50 PatchCore-style prototypes, using the same calibration draws and held-out evaluation sets as SPARC~\citep{li2026fastref}. 

\subsection{Experimental Environment}
\label{sec:experimental-environment}

All experiments were executed on three GPU compute nodes using a common, version-locked software environment. Experimental units were distributed through a shared queue, with each unit assigned to a single GPU. The nodes shared the same experiment code, datasets, model assets, and result storage.

\subsubsection{Hardware Specifications}
\begin{table*}[ht!]
\caption{\textbf{Hardware specifications.} System memory/cores show host RAM/CPUs; VRAM is per device.}
\label{tab:hardware-specifications}
\centering
\setlength{\tabcolsep}{5pt}
\resizebox{\textwidth}{!}{%
\begin{tabular}{@{}llllll@{}}
\toprule
Node
& GPU configuration
& VRAM
& CPU configuration
& Cores/Threads
& System Memory \\
\midrule
A
& $2\times$ NVIDIA L40
& 48 GB
& $2\times$ Intel Xeon Silver 4410Y
& 24/48
& 0.5 TB\\
B
& $8\times$ NVIDIA RTX 6000 Ada Generation
& 48 GB
& $2\times$ AMD EPYC 9124
& 32/64
& 1.5 TB \\
C
& $6\times$ NVIDIA RTX 6000 Ada Generation
& 48 GB
& $2\times$ AMD EPYC 9124
& 32/64
& 1.5 TB \\
\bottomrule
\end{tabular}
}%
\end{table*}

\Cref{tab:hardware-specifications} summarizes the hardware used for the experiments. All GPUs are based on the NVIDIA Ada Lovelace architecture and have compute capability 8.9. All nodes ran Ubuntu 24.04.4 LTS with glibc 2.39. Each node provided 8 GiB of swap space in addition to the system memory listed above.

\subsubsection{Software and Frameworks}
\begin{table}[!ht]
\caption{\textbf{Software environment.} Versions were fixed across all experimental nodes.}
\label{tab:software-environment}
\centering
\setlength{\tabcolsep}{7pt}
\begin{tabular}{@{}ll@{}}
\toprule
Library or framework & Version \\
\midrule
Python & 3.12.3 \\
PyTorch & 2.7.0+cu128 \\
Torchvision & 0.22.0+cu128 \\
xFormers & 0.0.30 \\
CUDA runtime/toolkit & 12.8/12.8.93 \\
cuDNN & 9.7.1 \\
Anomalib & 2.4.2 \\
TorchMetrics & 1.8.2 \\
Kornia & 0.8.1 \\
OpenCV & 4.11.0.86 \\
NumPy & 1.26.4 \\
SciPy & 1.16.1 \\
scikit-learn & 1.7.1 \\
scikit-image & 0.25.2 \\
Pillow & 11.3.0 \\
imagecodecs & 2024.12.30 \\
pandas & 2.2.3 \\
Matplotlib & 3.10.5 \\
Transformers & 4.56.2 \\
timm & 1.0.19 \\
FAISS GPU (CUDA 12) & 1.11.0 \\
Joblib & 1.5.1 \\
\bottomrule
\end{tabular}
\end{table}

All experiments used Python 3.12.3 and PyTorch 2.7.0 compiled for CUDA 12.8. The installed CUDA toolkit was version 12.8.93, and PyTorch used cuDNN 9.7.1. The complete environment was installed from a version-locked dependency specification and was identical across the three compute nodes. The principal libraries and frameworks are listed in \Cref{tab:software-environment}.

GPU execution used TF32 for supported floating-point matrix multiplication and convolution operations. Mixed-precision execution used bfloat16 in model paths for which it was enabled. Model checkpoints and third-party repositories were prepared in local caches before execution, and workload-time network access was restricted to ensure that every worker used the same model assets.

\section{Comparisons}
\label{sec:supp-baselines}
\subsection{Query-Conditioned Few-Shot Baselines}
\label{sec:supp-comparison}
\begin{table*}[!ht]
\caption{\textbf{Query-conditioned detector comparison.} At $k{=}8$, the table reports per-dataset Image AUROC and localization results; +SPARC denotes our correction layer.}
\label{tab:comparison-baselines-auroc}
\label{tab:comparison-baselines-aupro}
\centering
\setlength{\tabcolsep}{1.6pt}
\renewcommand{\arraystretch}{1.02}
\resizebox{\textwidth}{!}{%
\begin{tabular}{@{}l *{9}{r}@{}}
\toprule
& \multicolumn{3}{c}{\textbf{MVTec AD 2}} & \multicolumn{2}{c}{\textbf{AeBAD-S}} & \multicolumn{2}{c}{\textbf{VisA}} & \multicolumn{2}{c}{\textbf{RAD}} \\
\cmidrule(lr){2-4}\cmidrule(lr){5-6}\cmidrule(lr){7-8}\cmidrule(lr){9-10}
\textbf{Method} & \makecell{Image\\AUROC} & \mbox{AU-PRO$_{0.05}$} & \mbox{AU-PRO$_{0.3}$} & \makecell{Image\\AUROC} & \mbox{AU-PRO$_{0.3}$} & \makecell{Image\\AUROC} & \mbox{AU-PRO$_{0.3}$} & \makecell{Image\\AUROC} & \mbox{AU-PRO$_{0.3}$} \\
\midrule
FastRef & 83.7 & 22.7 & 46.4 & 68.0 & 82.4 & 83.7 & 76.4 & 98.1 & 90.3 \\
FastRecon & 75.5 & 17.4 & 41.6 & 57.1 & 78.8 & 72.0 & 63.2 & 98.2 & 86.7 \\
\midrule
SubspaceAD + SPARC & 85.8 & 45.6 & 69.3 & 88.1 & 92.7 & 98.2 & 95.9 & 96.4 & 94.6 \\
AnomalyDINO + SPARC & 82.0 & 38.8 & 63.7 & 73.5 & 86.7 & 93.7 & 92.0 & 97.0 & 90.5 \\
PaDiM + SPARC & 81.6 & 17.1 & 47.5 & 62.5 & 82.6 & 89.6 & 82.9 & 97.8 & 76.3 \\
PatchCore + SPARC & 83.6 & 20.4 & 45.7 & 66.3 & 86.3 & 93.9 & 86.5 & 96.2 & 77.9 \\
MuSc + SPARC & 84.5 & 29.9 & 56.7 & 39.0 & 77.6 & 87.6 & 88.5 & 62.3 & 87.4 \\
WinCLIP + SPARC & 79.3 & 17.7 & 47.1 & 76.4 & 82.6 & 76.6 & 74.4 & 83.4 & 56.8 \\
AnomalyCLIP$_{\mathrm{V}}$ + SPARC & 52.5 & 25.5 & 45.0 & 78.9 & 86.3 & 82.1 & 83.2 & 79.6 & 76.6 \\
AnomalyCLIP$_{\mathrm{M}}$ + SPARC & 58.3 & 25.5 & 43.8 & 75.2 & 89.2 & 88.9 & 87.1 & 78.6 & 67.6 \\
SPADE + SPARC & 71.0 & 19.5 & 46.6 & 43.6 & 89.4 & 84.2 & 86.0 & 34.7 & 79.3 \\
\bottomrule
\end{tabular}
}%
\end{table*}

We situate SPARC against two dedicated query-conditioned few-shot methods, FastRecon~\citep{fang2023fastrecon} and FastRef~\citep{li2026fastref}, each run in its intended regime (\Cref{tab:comparison-baselines-auroc}; see the Experiments section of the main paper for the setup).

This is not an information-matched head-to-head.
SPARC is a correction layer rather than a detector, so its host, PatchCore, additionally consumes the full train set that these standalone methods never see.
The comparison can therefore support only the weaker statement that an off-the-shelf detector equipped with SPARC reaches the level of purpose-built few-shot detectors.

Acknowledging that, SPARC+PatchCore reaches comparable absolute scores to both across categories.
We include these numbers only as a cross-setting reference point.
The controlled evidence for SPARC's contribution instead lies in the within-detector toggle and correction-mode ablation, which hold the host, its training access, and the $k$ draws fixed and isolate the correction as this cross-setting comparison cannot.

\subsection{Global Feature Alignment (CORAL)}
\label{sec:supp-coral}
As a counterpart to SPARC's per-cell projection, we implemented CORAL~\citep{sun2016coral}.

On $\ell_2$-normalized features we map a deployment feature by $f\mapsto(f-\mu_s)\,C_s^{-1/2}C_t^{1/2}+\mu_t$, with source covariance $C_s$ from the deployment normals, target covariance $C_t$ from the detector's own train/good features, and $C_{\{s,t\}}=\mathrm{Cov}(\cdot)+\lambda I$ ($\lambda{=}1$). 
We run two information budgets that differ only in how many deployment normals estimate $C_s$.
\emph{CORAL-fs} uses the matched few-shot budget ($k{\le}8$).
\emph{CORAL full-normal} uses all ground-truth normal deployment images, including normals in the evaluation split, and serves as a transductive diagnostic rather than a deployable baseline.
This applies only to the six detectors with a native train-derived reference; MuSc and AnomalyCLIP are excluded as they provide no such target covariance under the same budget. 

\Cref{tab:coral-sectioned} reports both budgets under the same paired protocol.
CORAL's pooled effects are detector-dependent: Image AUROC changes remain near zero for AnomalyDINO, SubspaceAD, and WinCLIP, while PatchCore and PaDiM decline; pooled AU-PRO$_{0.3}$ also declines for PatchCore, PaDiM, and SubspaceAD (\Cref{tab:coral-budget-sweep}).
Except for SPADE Image AUROC, CORAL-fs and CORAL full-normal agree closely across nearly all cells.
Their similarity for most detectors suggests that few-shot covariance estimation alone does not explain CORAL's behavior.

This distinction motivates SPARC.
Rather than rescaling the whole feature space, SPARC removes an at-most-$(k{-}1)$-dimensional calibration-disagreement subspace and leaves its orthogonal complement untouched.

\begin{table*}[!ht]
\caption{\textbf{CORAL results by detector and dataset.} Matched-budget and full-normal diagnostic conditions use the same paired evaluation protocol.}
\label{tab:coral-sectioned}
\centering
\setlength{\tabcolsep}{1.4pt}
\renewcommand{\arraystretch}{0.92}
\resizebox{\textwidth}{!}{%
\begin{tabular}{@{}ll|rrr|rrr|rrr|rrr@{}}
\toprule
&& \multicolumn{3}{c|}{\textbf{MVTec AD~2}} & \multicolumn{3}{c|}{\textbf{AeBAD-S}} & \multicolumn{3}{c|}{\textbf{VisA}} & \multicolumn{3}{c}{\textbf{RAD}} \\
\cmidrule(lr){3-5}\cmidrule(lr){6-8}\cmidrule(lr){9-11}\cmidrule(lr){12-14}
\textbf{Method} & \textbf{Budget} & base & +C & $\Delta$ & base & +C & $\Delta$ & base & +C & $\Delta$ & base & +C & $\Delta$ \\
\midrule
\multicolumn{14}{@{}l}{\textit{Image AUROC}} \\
SubspaceAD & CORAL-fs & 70.9 & \textbf{71.0} & +0.2 & 73.6 & \textbf{77.2} & +3.6 & \textbf{98.1} & 97.6 & -0.5 & \textbf{96.3} & 95.6 & -0.7 \\
SubspaceAD & CORAL full-normal & 70.9 & \textbf{71.1} & +0.2 & 73.6 & \textbf{77.3} & +3.7 & \textbf{98.1} & 97.6 & -0.5 & \textbf{96.3} & 95.6 & -0.7 \\
AnomalyDINO & CORAL-fs & 71.7 & \textbf{72.5} & +0.8 & 71.7 & \textbf{73.4} & +1.7 & 95.3 & 95.3 & +0.0 & 97.2 & \textbf{97.4} & +0.2 \\
AnomalyDINO & CORAL full-normal & 71.7 & \textbf{72.7} & +1.0 & 71.7 & \textbf{73.5} & +1.9 & 95.3 & \textbf{95.4} & +0.1 & 97.2 & \textbf{97.4} & +0.2 \\
PaDiM & CORAL-fs & \textbf{59.9} & 54.9 & -5.0 & 61.6 & \textbf{64.0} & +2.4 & \textbf{89.3} & 82.1 & -7.2 & 97.5 & \textbf{98.0} & +0.5 \\
PaDiM & CORAL full-normal & \textbf{59.9} & 54.9 & -5.1 & 61.6 & \textbf{64.2} & +2.7 & \textbf{89.3} & 82.1 & -7.2 & 97.5 & \textbf{98.1} & +0.6 \\
PatchCore & CORAL-fs & \textbf{68.7} & 65.6 & -3.1 & 64.1 & \textbf{66.7} & +2.6 & \textbf{94.3} & 91.7 & -2.6 & 92.1 & \textbf{93.6} & +1.5 \\
PatchCore & CORAL full-normal & \textbf{68.7} & 65.6 & -3.1 & 64.1 & \textbf{66.7} & +2.6 & \textbf{94.3} & 91.7 & -2.6 & 92.1 & \textbf{93.5} & +1.4 \\
WinCLIP & CORAL-fs & 62.9 & \textbf{63.1} & +0.2 & 75.8 & \textbf{76.4} & +0.6 & 75.4 & 75.4 & -0.1 & 80.5 & \textbf{82.6} & +2.2 \\
WinCLIP & CORAL full-normal & 62.9 & \textbf{63.3} & +0.4 & 75.8 & \textbf{76.5} & +0.6 & 75.4 & 75.4 & -0.0 & 80.5 & \textbf{82.6} & +2.1 \\
SPADE & CORAL-fs & 53.8 & \textbf{61.2} & +7.4 & 38.3 & \textbf{50.6} & +12.2 & \textbf{83.7} & 82.1 & -1.6 & 28.9 & \textbf{91.4} & +62.5 \\
SPADE & CORAL full-normal & 53.8 & \textbf{66.1} & +12.3 & 38.3 & \textbf{51.9} & +13.6 & 83.7 & \textbf{84.8} & +1.1 & 28.9 & \textbf{92.7} & +63.8 \\
\midrule
\multicolumn{14}{@{}l}{\textit{AU-PRO$_{0.3}$}} \\
SubspaceAD & CORAL-fs & \textbf{61.3} & 59.4 & -1.9 & \textbf{89.3} & 89.2 & -0.1 & \textbf{96.1} & 96.0 & -0.1 & \textbf{94.0} & 93.7 & -0.3 \\
SubspaceAD & CORAL full-normal & \textbf{61.3} & 59.3 & -2.0 & \textbf{89.3} & 89.2 & -0.1 & \textbf{96.1} & 96.0 & -0.1 & \textbf{94.0} & 93.7 & -0.3 \\
AnomalyDINO & CORAL-fs & 59.9 & \textbf{60.7} & +0.8 & 83.9 & \textbf{86.8} & +2.9 & 92.4 & 92.4 & +0.0 & 89.4 & 89.4 & +0.0 \\
AnomalyDINO & CORAL full-normal & 59.9 & \textbf{60.7} & +0.8 & 83.9 & \textbf{86.9} & +3.0 & 92.4 & 92.4 & +0.0 & 89.4 & 89.4 & +0.0 \\
PaDiM & CORAL-fs & \textbf{41.0} & 38.1 & -2.9 & \textbf{82.0} & 77.9 & -4.1 & \textbf{82.8} & 73.9 & -8.9 & 75.5 & \textbf{78.4} & +2.9 \\
PaDiM & CORAL full-normal & \textbf{41.0} & 38.2 & -2.8 & \textbf{82.0} & 78.0 & -4.0 & \textbf{82.8} & 73.9 & -8.9 & 75.5 & \textbf{78.4} & +2.9 \\
PatchCore & CORAL-fs & \textbf{42.8} & 36.4 & -6.4 & \textbf{85.1} & 79.8 & -5.3 & \textbf{87.8} & 81.0 & -6.8 & 73.2 & \textbf{77.8} & +4.6 \\
PatchCore & CORAL full-normal & \textbf{42.8} & 36.3 & -6.5 & \textbf{85.1} & 80.0 & -5.2 & \textbf{87.8} & 81.0 & -6.8 & 73.2 & \textbf{77.7} & +4.4 \\
WinCLIP & CORAL-fs & 45.0 & \textbf{45.2} & +0.2 & 82.6 & \textbf{82.7} & +0.1 & \textbf{73.6} & 73.5 & -0.0 & 55.6 & \textbf{60.4} & +4.8 \\
WinCLIP & CORAL full-normal & 45.0 & \textbf{45.2} & +0.3 & 82.6 & \textbf{82.7} & +0.1 & 73.6 & 73.6 & -0.0 & 55.6 & \textbf{60.5} & +4.8 \\
SPADE & CORAL-fs & 42.0 & \textbf{44.2} & +2.2 & \textbf{89.0} & 87.2 & -1.7 & \textbf{87.9} & 85.0 & -2.9 & 79.9 & \textbf{82.8} & +2.9 \\
SPADE & CORAL full-normal & 42.0 & \textbf{44.6} & +2.6 & \textbf{89.0} & 87.2 & -1.7 & \textbf{87.9} & 85.0 & -2.9 & 79.9 & \textbf{82.7} & +2.8 \\
\midrule
\multicolumn{14}{@{}l}{\textit{AU-PRO$_{0.05}$}} \\
SubspaceAD & CORAL-fs & \textbf{35.4} & 33.7 & -1.7 & \multicolumn{3}{c|}{--} & \multicolumn{3}{c|}{--} & \multicolumn{3}{c}{--} \\
SubspaceAD & CORAL full-normal & \textbf{35.4} & 33.6 & -1.9 & \multicolumn{3}{c|}{--} & \multicolumn{3}{c|}{--} & \multicolumn{3}{c}{--} \\
AnomalyDINO & CORAL-fs & 37.0 & \textbf{37.2} & +0.3 & \multicolumn{3}{c|}{--} & \multicolumn{3}{c|}{--} & \multicolumn{3}{c}{--} \\
AnomalyDINO & CORAL full-normal & 37.0 & \textbf{37.2} & +0.3 & \multicolumn{3}{c|}{--} & \multicolumn{3}{c|}{--} & \multicolumn{3}{c}{--} \\
PaDiM & CORAL-fs & \textbf{11.3} & 9.1 & -2.2 & \multicolumn{3}{c|}{--} & \multicolumn{3}{c|}{--} & \multicolumn{3}{c}{--} \\
PaDiM & CORAL full-normal & \textbf{11.3} & 9.1 & -2.2 & \multicolumn{3}{c|}{--} & \multicolumn{3}{c|}{--} & \multicolumn{3}{c}{--} \\
PatchCore & CORAL-fs & \textbf{19.0} & 12.3 & -6.7 & \multicolumn{3}{c|}{--} & \multicolumn{3}{c|}{--} & \multicolumn{3}{c}{--} \\
PatchCore & CORAL full-normal & \textbf{19.0} & 12.2 & -6.8 & \multicolumn{3}{c|}{--} & \multicolumn{3}{c|}{--} & \multicolumn{3}{c}{--} \\
WinCLIP & CORAL-fs & 17.1 & \textbf{17.2} & +0.2 & \multicolumn{3}{c|}{--} & \multicolumn{3}{c|}{--} & \multicolumn{3}{c}{--} \\
WinCLIP & CORAL full-normal & 17.1 & \textbf{17.3} & +0.2 & \multicolumn{3}{c|}{--} & \multicolumn{3}{c|}{--} & \multicolumn{3}{c}{--} \\
SPADE & CORAL-fs & 18.4 & \textbf{19.8} & +1.3 & \multicolumn{3}{c|}{--} & \multicolumn{3}{c|}{--} & \multicolumn{3}{c}{--} \\
SPADE & CORAL full-normal & 18.4 & \textbf{19.8} & +1.4 & \multicolumn{3}{c|}{--} & \multicolumn{3}{c|}{--} & \multicolumn{3}{c}{--} \\
\bottomrule
\end{tabular}
}%
\end{table*}

\begin{table*}[!ht]
\caption{\textbf{CORAL calibration budget.} Few-shot estimates use $k\in\{2,4,8\}$. The full-normal diagnostic uses all ground-truth normal deployment images, including normals in the evaluation subset.}
\label{tab:coral-budget-sweep}
\centering
\setlength{\tabcolsep}{5pt}
\renewcommand{\arraystretch}{1.02}
\begin{tabular}{@{}lrrrrr@{}}
\toprule
\textbf{Method} & base & $\Delta C_{fs,k=2}$ & $\Delta C_{fs,k=4}$ & $\Delta C_{fs,k=8}$ & $\Delta C_{full}$ \\
\midrule
\multicolumn{6}{@{}l}{\textit{Image AUROC}} \\
SubspaceAD & 84.7 & +0.4 & +0.5 & +0.7 & +0.7 \\
AnomalyDINO & 84.0 & +0.5 & +0.6 & +0.7 & +0.8 \\
PaDiM & 77.1 & -2.4 & -2.2 & -2.3 & -2.3 \\
PatchCore & 79.8 & -1.0 & -0.7 & -0.4 & -0.4 \\
WinCLIP & 73.7 & +0.6 & +0.8 & +0.7 & +0.8 \\
SPADE & 51.2 & +16.3 & +19.9 & +20.2 & +22.7 \\
\midrule
\multicolumn{6}{@{}l}{\textit{AU-PRO$_{0.3}$}} \\
SubspaceAD & 85.2 & -0.6 & -0.6 & -0.6 & -0.6 \\
AnomalyDINO & 81.4 & +0.9 & +0.8 & +0.9 & +0.9 \\
PaDiM & 70.3 & -3.2 & -3.1 & -3.3 & -3.2 \\
PatchCore & 72.2 & -3.8 & -3.4 & -3.5 & -3.5 \\
WinCLIP & 64.2 & +1.1 & +1.4 & +1.3 & +1.3 \\
SPADE & 74.7 & +0.1 & +0.2 & +0.1 & +0.2 \\
\midrule
\multicolumn{6}{@{}l}{\textit{AU-PRO$_{0.05}$}} \\
SubspaceAD & 35.4 & -1.8 & -1.7 & -1.7 & -1.9 \\
AnomalyDINO & 37.0 & +0.4 & +0.2 & +0.3 & +0.3 \\
PaDiM & 11.3 & -2.2 & -2.2 & -2.2 & -2.2 \\
PatchCore & 19.0 & -6.8 & -6.9 & -6.7 & -6.8 \\
WinCLIP & 17.1 & +0.4 & +0.4 & +0.2 & +0.2 \\
SPADE & 18.4 & +1.2 & +1.3 & +1.3 & +1.4 \\
\bottomrule
\end{tabular}
\end{table*}

\subsection{Normalization-Statistic Adaptation}
\label{sec:supp-bn}

We adapt the batch-normalization (BN) statistics of the common WRN-50 encoder used by PaDiM, PatchCore, and SPADE~\citep{schneider2020improving}.
The calibration-only variant uses statistics estimated from the same $k{=}8$ deployment normals as SPARC.
The source-mixed variant combines these statistics with the source running statistics using a pseudo-count of $N{=}16$.
We also include a transductive full-normal diagnostic estimated from all ground-truth normal deployment images, including normals in the evaluation subset.

\begin{sidewaystable*}[p]
\caption{\textbf{Detailed feature-level adaptation controls.}
Results use PaDiM, PatchCore, and SPADE with a common WRN-50 encoder at $k{=}8$.
All rows are paired changes in percentage points (pp).
Values are category macro-averages per seed and are reported as mean$\pm$sample SD over five seeds.
For BN adaptation, $N$ is the source-statistics pseudo-count.
$N{=}0$ uses only calibration statistics, while $N{=}16$ blends source and calibration statistics.
$^\dagger$Full-normal diagnostics use all ground-truth normal deployment images, including normals in the evaluation subset, and are nondeployable transductive controls.}
\label{tab:normalization-comparison-detailed}
\centering
\setlength{\tabcolsep}{1.8pt}
\renewcommand{\arraystretch}{1.02}
\resizebox{\textheight}{!}{%
\begin{tabular}{@{}ll rrr rr rr rr@{}}
\toprule
& & \multicolumn{3}{c}{MVTec AD~2} & \multicolumn{2}{c}{AeBAD-S} & \multicolumn{2}{c}{VisA} & \multicolumn{2}{c}{RAD} \\
\cmidrule(lr){3-5}
\cmidrule(lr){6-7}
\cmidrule(lr){8-9}
\cmidrule(lr){10-11}
Method
& Variant
& Image AUROC
& AU-PRO$_{0.05}$
& AU-PRO$_{0.3}$
& Image AUROC
& AU-PRO$_{0.3}$
& Image AUROC
& AU-PRO$_{0.3}$
& Image AUROC
& AU-PRO$_{0.3}$ \\
\midrule

PaDiM \citep{defard2021padim}
& CORAL few-shot
& $-$5.0$\pm$0.6
& $-$2.2$\pm$0.1
& $-$2.9$\pm$0.2
& +2.4$\pm$0.6
& $-$4.1$\pm$0.2
& $-$7.2$\pm$0.2
& $-$8.9$\pm$0.1
& +0.5$\pm$0.5
& +2.9$\pm$0.6 \\

& \shortstack[l]{Calibration-only BN}
& +3.1$\pm$0.9
& +7.3$\pm$0.3
& +2.5$\pm$0.2
& $-$8.0$\pm$0.6
& $-$2.0$\pm$0.3
& $-$0.1$\pm$0.2
& $-$0.6$\pm$0.1
& $-$2.8$\pm$2.6
& $-$10.5$\pm$2.6 \\

& \shortstack[l]{Source-mixed BN}
& +4.6$\pm$0.5
& +6.4$\pm$0.1
& +3.5$\pm$0.1
& $-$0.3$\pm$0.3
& +1.0$\pm$0.1
& 0.0$\pm$0.1
& $-$0.1$\pm$0.0
& $-$1.5$\pm$0.3
& +4.8$\pm$0.1 \\

& CORAL full-normal$^\dagger$
& $-$5.1$\pm$0.7
& $-$2.2$\pm$0.0
& $-$2.8$\pm$0.0
& +2.7$\pm$0.1
& $-$4.0$\pm$0.0
& $-$7.2$\pm$0.2
& $-$8.9$\pm$0.0
& +0.6$\pm$0.4
& +2.9$\pm$0.0 \\

& BN full-normal$^\dagger$
& +3.8$\pm$0.4
& +7.7$\pm$0.0
& +2.7$\pm$0.0
& $-$7.3$\pm$0.2
& $-$1.3$\pm$0.0
& +0.1$\pm$0.2
& $-$0.5$\pm$0.0
& $-$0.2$\pm$0.5
& $-$8.7$\pm$0.0 \\

& SPARC (ours)
& +21.7$\pm$2.9
& +5.8$\pm$3.1
& +6.5$\pm$1.2
& +1.0$\pm$0.5
& +0.6$\pm$0.1
& +0.3$\pm$0.1
& +0.1$\pm$0.1
& +0.3$\pm$0.1
& +0.9$\pm$0.1 \\

\addlinespace[2pt]
\midrule

PatchCore \citep{roth2022patchcore}
& CORAL few-shot
& $-$3.1$\pm$0.5
& $-$6.7$\pm$0.3
& $-$6.4$\pm$0.2
& +2.6$\pm$1.1
& $-$5.3$\pm$0.3
& $-$2.6$\pm$0.1
& $-$6.8$\pm$0.0
& +1.5$\pm$0.4
& +4.6$\pm$1.0 \\

& \shortstack[l]{Calibration-only BN}
& $-$1.3$\pm$1.8
& +2.4$\pm$0.3
& +2.4$\pm$0.4
& $-$9.3$\pm$1.3
& $-$3.0$\pm$1.1
& +0.8$\pm$0.1
& $-$0.3$\pm$0.1
& +2.7$\pm$0.6
& $-$0.8$\pm$2.1 \\

& \shortstack[l]{Source-mixed BN}
& +1.8$\pm$1.3
& +5.2$\pm$0.7
& +4.4$\pm$0.6
& +2.4$\pm$0.8
& +1.1$\pm$0.4
& +0.5$\pm$0.2
& 0.0$\pm$0.2
& +0.3$\pm$0.7
& +4.4$\pm$0.4 \\

& CORAL full-normal$^\dagger$
& $-$3.1$\pm$0.4
& $-$6.8$\pm$0.1
& $-$6.5$\pm$0.0
& +2.6$\pm$0.4
& $-$5.2$\pm$0.0
& $-$2.6$\pm$0.1
& $-$6.8$\pm$0.0
& +1.4$\pm$0.1
& +4.4$\pm$0.1 \\

& BN full-normal$^\dagger$
& +0.2$\pm$0.5
& +2.4$\pm$0.1
& +1.9$\pm$0.1
& $-$7.2$\pm$0.3
& $-$2.0$\pm$0.0
& +0.5$\pm$0.1
& $-$0.4$\pm$0.0
& +3.2$\pm$0.3
& $-$0.4$\pm$0.1 \\

& SPARC (ours)
& +14.8$\pm$2.0
& +1.3$\pm$1.8
& +2.9$\pm$1.9
& +2.2$\pm$1.3
& +1.2$\pm$0.2
& $-$0.3$\pm$0.3
& $-$1.3$\pm$0.3
& +4.0$\pm$0.4
& +4.7$\pm$1.3 \\

\addlinespace[2pt]
\midrule

SPADE \citep{cohen2020spade}
& CORAL few-shot
& +7.4$\pm$1.9
& +1.3$\pm$0.1
& +2.2$\pm$0.3
& +12.2$\pm$2.9
& $-$1.7$\pm$0.0
& $-$1.6$\pm$1.7
& $-$2.9$\pm$0.0
& +62.5$\pm$2.2
& +2.9$\pm$0.3 \\

& \shortstack[l]{Calibration-only BN}
& +4.4$\pm$0.7
& +3.6$\pm$0.1
& +7.8$\pm$0.1
& +0.1$\pm$1.4
& $-$1.0$\pm$0.1
& +0.2$\pm$0.4
& $-$1.1$\pm$0.1
& +14.8$\pm$4.1
& +1.7$\pm$1.7 \\

& \shortstack[l]{Source-mixed BN}
& $-$0.5$\pm$0.4
& +6.7$\pm$0.1
& +8.9$\pm$0.0
& $-$0.1$\pm$0.8
& +0.3$\pm$0.1
& +0.5$\pm$0.2
& +0.2$\pm$0.0
& +1.1$\pm$0.4
& +11.3$\pm$0.2 \\

& CORAL full-normal$^\dagger$
& +12.3$\pm$0.7
& +1.4$\pm$0.0
& +2.6$\pm$0.0
& +13.6$\pm$0.4
& $-$1.7$\pm$0.0
& +1.1$\pm$0.1
& $-$2.9$\pm$0.0
& +63.8$\pm$0.8
& +2.8$\pm$0.0 \\

& BN full-normal$^\dagger$
& +4.8$\pm$1.0
& +3.9$\pm$0.0
& +7.9$\pm$0.0
& +0.9$\pm$0.3
& $-$0.7$\pm$0.0
& +0.6$\pm$0.1
& $-$1.1$\pm$0.0
& +17.5$\pm$0.5
& +1.6$\pm$0.0 \\

& SPARC (ours)
& +17.2$\pm$2.3
& +1.1$\pm$0.1
& +4.6$\pm$0.7
& +5.2$\pm$1.0
& +0.5$\pm$0.1
& +0.5$\pm$0.6
& $-$1.9$\pm$0.1
& +5.8$\pm$0.8
& $-$0.6$\pm$0.2 \\
\bottomrule
\end{tabular}%
}
\end{sidewaystable*}

BN produces a metric-dependent tradeoff.
SPARC provides the largest shift-prone Image AUROC gain, while source-mixed BN achieves the largest aggregate localization gains.
The detailed results in \Cref{tab:normalization-comparison-detailed} show that the BN localization gains are concentrated in particular detectors and datasets, especially SPADE, rather than being uniform across the three detectors.
The full-normal diagnostic does not remove this detector dependence, suggesting that the tradeoff is not explained solely by noisy few-shot statistics.

\section{Design Ablations}
\label{sec:supp-design}
\subsection{Correction-Mode Ablation}
\label{sec:supp-ab1-extended}
\begin{table*}[!ht]
\caption{\textbf{Correction-mode ablation.} \emph{base}: no correction; \emph{mean-only}: per-cell centroid subtraction; \emph{SPARC (global subspace)}: a single rank-$(k{-}1)$ subspace pooled over all cells; \emph{SPARC (per-cell)}: our method. \textbf{Bold} marks the best entry per column; Pooled~$\Delta$ is the change vs.\ base. Image AUROC and AU-PRO$_{0.3}$ pool all 28 dataset--category pairs; AU-PRO$_{0.05}$ pools the eight MVTec AD~2 categories only. All results use $k{=}8$.}
\label{tab:ab1}
\centering
\setlength{\tabcolsep}{1.5pt}
\renewcommand{\arraystretch}{1.08}
\resizebox{\textwidth}{!}{%
\begin{tabular}{@{}l c c c c c c c c c c@{}}
\toprule
Variant & PatchCore & PaDiM & SPADE & \shortstack{Anomaly\\DINO} & \shortstack{Subspace\\AD} & MuSc & \shortstack{Anomaly\\CLIP$_\mathrm{V}$} & \shortstack{Anomaly\\CLIP$_\mathrm{M}$} & Pooled & \shortstack{Pooled\\$\Delta$} \\
\midrule
\multicolumn{11}{l}{\textit{Image AUROC}} \\
base & 82.4 & 78.1 & 60.8 & 85.5 & 86.5 & 66.6 & \textbf{72.8} & \textbf{76.7} & 76.2 & 0.0 \\
mean-only & 86.4 & 78.1 & 60.9 & 64.1 & 92.3 & 25.3 & \textbf{72.8} & \textbf{76.7} & 69.6 & $-$6.6 \\
SPARC (global subspace) & 84.0 & 78.9 & 63.8 & 84.5 & 88.3 & 66.5 & \textbf{72.8} & \textbf{76.7} & 76.9 & +0.8 \\
SPARC (per-cell, ours) & \textbf{87.3} & \textbf{84.6} & \textbf{67.5} & \textbf{87.9} & \textbf{92.9} & \textbf{76.2} & \textbf{72.8} & \textbf{76.7} & \textbf{80.8} & \textbf{+4.6} \\
\midrule
\multicolumn{11}{l}{\textit{AU-PRO$_{0.3}$}} \\
base & 72.5 & 69.7 & 73.8 & 81.5 & 84.9 & \textbf{78.1} & \textbf{72.2} & 71.9 & 75.6 & 0.0 \\
mean-only & 72.7 & 69.7 & 74.0 & 37.1 & 87.1 & 29.8 & 52.7 & 59.6 & 60.4 & $-$15.2 \\
SPARC (global subspace) & \textbf{73.8} & 70.9 & \textbf{74.3} & 77.2 & 85.9 & 77.1 & 67.3 & 68.7 & 74.4 & $-$1.2 \\
SPARC (per-cell, ours) & 73.6 & \textbf{71.8} & \textbf{74.3} & \textbf{82.9} & \textbf{87.7} & 77.7 & 71.8 & \textbf{72.2} & \textbf{76.5} & \textbf{+0.9} \\
\midrule
\multicolumn{11}{l}{\textit{AU-PRO$_{0.05}$ (MVTec AD~2)}} \\
base & 19.0 & 11.3 & 18.4 & 37.0 & 35.4 & 26.6 & 25.2 & 25.2 & 24.8 & 0.0 \\
mean-only & 17.2 & 11.3 & \textbf{19.5} & 6.3 & 44.5 & 2.5 & 9.0 & 10.2 & 15.1 & $-$9.7 \\
SPARC (global subspace) & 20.3 & 15.2 & 18.6 & 33.0 & 39.0 & 27.3 & 22.5 & 23.5 & 24.9 & +0.1 \\
SPARC (per-cell, ours) & \textbf{20.4} & \textbf{17.1} & \textbf{19.5} & \textbf{38.8} & \textbf{45.6} & \textbf{29.9} & \textbf{25.5} & \textbf{25.5} & \textbf{27.8} & \textbf{+3.0} \\
\bottomrule
\end{tabular}
}%
\end{table*}

\Cref{tab:ab1} reports the correction-mode ablation pooled across all 28 dataset--category pairs at $k{=}8$.
Per-cell SPARC (our method) is the only mode with consistent pooled gains on all three metrics, whereas per-cell mean subtraction alone degrades every metric and a single global subspace is essentially neutral.
Why per-cell subspace removal beats mean subtraction is analyzed in \Cref{sec:supp-mechanism}.

\subsection{Projection-Rank Ablation}
\label{sec:supp-rank}
\Cref{tab:supp-ab2-rank} reports the full per-detector rank sweep. 
Across MVTec AD~2, the gains generally increase with rank and approach a plateau near the algebraic saturation rank, although the exact empirical optimum varies by detector and metric.
On shift-free VisA, the changes remain small and mixed.
We therefore use $r{=}k{-}1$ as a validation-free algebraic saturation choice rather than selecting it through significance testing.

\begin{table*}[!ht]
\caption{\textbf{Projection-rank sensitivity.} Paired changes as the retained rank approaches $k{-}1$.}
\label{tab:supp-ab2-rank}
\centering
\setlength{\tabcolsep}{1.65pt}
\renewcommand{\arraystretch}{0.96}
\resizebox{\textwidth}{!}{%
\begin{tabular}{@{}l|rrrrrrr|rrrrrrr@{}}
\toprule
& \multicolumn{7}{c|}{\textbf{MVTec AD~2}} & \multicolumn{7}{c}{\textbf{VisA}} \\
\cmidrule(lr){2-8}\cmidrule(lr){9-15}
\textbf{Method} & $r{=}1$ & $r{=}2$ & $r{=}3$ & $r{=}4$ & $r{=}5$ & $r{=}6$ & $r{=}7$ & $r{=}1$ & $r{=}2$ & $r{=}3$ & $r{=}4$ & $r{=}5$ & $r{=}6$ & $r{=}7$ \\
\midrule
\multicolumn{15}{@{}l}{\textit{Image AUROC $\Delta$}} \\
SubspaceAD & +10.2 & +13.1 & +14.1 & +14.5 & +14.8 & +14.9 & +14.9 & +0.1 & +0.1 & +0.1 & +0.1 & +0.1 & +0.1 & +0.1 \\
AnomalyDINO & +2.7 & +5.8 & +8.0 & +9.3 & +9.9 & +10.2 & +10.3 & -0.2 & -0.6 & -1.0 & -1.2 & -1.4 & -1.5 & -1.6 \\
PaDiM & +12.3 & +18.9 & +21.1 & +21.4 & +21.6 & +21.6 & +21.7 & +0.2 & +0.2 & +0.3 & +0.3 & +0.3 & +0.3 & +0.3 \\
PatchCore & +5.0 & +10.1 & +13.7 & +14.8 & +15.3 & +15.5 & +14.8 & +0.3 & +0.0 & +0.0 & -0.2 & -0.2 & -0.1 & -0.3 \\
MuSc & +14.7 & +25.8 & +30.8 & +32.3 & +32.8 & +33.0 & +33.1 & -1.1 & -1.4 & -1.5 & -1.5 & -1.5 & -1.4 & -1.4 \\
WinCLIP & +6.7 & +11.3 & +14.1 & +15.5 & +16.0 & +16.1 & +16.4 & +0.9 & +1.1 & +1.1 & +1.1 & +1.1 & +1.1 & +1.1 \\
AnomalyCLIP$_{\mathrm{V}}$ & +0.0 & +0.0 & +0.0 & +0.0 & +0.0 & +0.0 & +0.0 & +0.0 & +0.0 & +0.0 & +0.0 & +0.0 & +0.0 & +0.0 \\
AnomalyCLIP$_{\mathrm{M}}$ & +0.0 & +0.0 & +0.0 & +0.0 & +0.0 & +0.0 & +0.0 & +0.0 & +0.0 & +0.0 & +0.0 & +0.0 & +0.0 & +0.0 \\
SPADE & +5.7 & +9.8 & +13.1 & +15.2 & +16.1 & +16.7 & +17.2 & -0.1 & +0.3 & +0.3 & +0.4 & +0.5 & +0.6 & +0.5 \\
\midrule
\multicolumn{15}{@{}l}{\textit{AU-PRO$_{0.3}$ $\Delta$}} \\
SubspaceAD & +5.7 & +7.0 & +7.3 & +7.6 & +7.7 & +7.9 & +8.0 & -0.0 & -0.1 & -0.1 & -0.1 & -0.1 & -0.1 & -0.2 \\
AnomalyDINO & +0.1 & +1.0 & +2.0 & +2.8 & +3.3 & +3.7 & +3.8 & -0.0 & -0.2 & -0.3 & -0.3 & -0.3 & -0.4 & -0.4 \\
PaDiM & +4.6 & +5.7 & +6.2 & +6.3 & +6.4 & +6.5 & +6.5 & +0.1 & +0.1 & +0.1 & +0.1 & +0.1 & +0.1 & +0.1 \\
PatchCore & +0.2 & +1.6 & +2.1 & +2.8 & +2.4 & +3.6 & +2.9 & -0.7 & -0.7 & -0.9 & -1.1 & -1.2 & -1.1 & -1.3 \\
MuSc & +1.0 & +2.3 & +3.3 & +3.6 & +3.7 & +3.7 & +3.7 & -1.5 & -2.0 & -2.3 & -2.5 & -2.6 & -2.6 & -2.7 \\
WinCLIP & +0.1 & +0.8 & +1.5 & +1.9 & +2.0 & +2.0 & +2.1 & +0.5 & +0.6 & +0.7 & +0.7 & +0.8 & +0.8 & +0.8 \\
AnomalyCLIP$_{\mathrm{V}}$ & -0.1 & +0.1 & +0.0 & +0.0 & +0.0 & +0.0 & -0.1 & -0.3 & -0.2 & -0.3 & -0.2 & -0.2 & -0.2 & -0.2 \\
AnomalyCLIP$_{\mathrm{M}}$ & +0.6 & +1.3 & +1.4 & +1.4 & +1.4 & +1.3 & +1.2 & -0.2 & -0.1 & -0.1 & -0.0 & +0.0 & +0.1 & +0.1 \\
SPADE & +1.0 & +2.5 & +3.5 & +4.1 & +4.4 & +4.5 & +4.6 & -0.6 & -1.1 & -1.4 & -1.6 & -1.7 & -1.8 & -1.9 \\
\midrule
\multicolumn{15}{@{}l}{\textit{AU-PRO$_{0.05}$ $\Delta$}} \\
SubspaceAD & +8.2 & +9.3 & +9.7 & +9.8 & +9.9 & +10.1 & +10.2 & \multicolumn{7}{c}{--} \\
AnomalyDINO & +0.5 & +0.7 & +1.2 & +1.6 & +1.8 & +1.8 & +1.8 & \multicolumn{7}{c}{--} \\
PaDiM & +5.1 & +5.4 & +5.6 & +5.7 & +5.7 & +5.8 & +5.8 & \multicolumn{7}{c}{--} \\
PatchCore & +0.7 & +0.9 & +1.6 & +1.5 & +1.8 & +2.0 & +1.3 & \multicolumn{7}{c}{--} \\
MuSc & +1.4 & +2.0 & +2.6 & +2.9 & +3.1 & +3.2 & +3.2 & \multicolumn{7}{c}{--} \\
WinCLIP & -0.1 & +0.2 & +0.4 & +0.5 & +0.6 & +0.6 & +0.7 & \multicolumn{7}{c}{--} \\
AnomalyCLIP$_{\mathrm{V}}$ & +0.2 & +0.3 & +0.4 & +0.4 & +0.4 & +0.3 & +0.3 & \multicolumn{7}{c}{--} \\
AnomalyCLIP$_{\mathrm{M}}$ & +0.3 & +0.4 & +0.4 & +0.5 & +0.5 & +0.4 & +0.4 & \multicolumn{7}{c}{--} \\
SPADE & +0.4 & +0.7 & +0.9 & +1.0 & +1.0 & +1.0 & +1.1 & \multicolumn{7}{c}{--} \\
\bottomrule
\end{tabular}
}%
\end{table*}

\subsection{Symmetric-Application Ablation}
\label{sec:supp-symmetric}
\Cref{tab:ab5} compares symmetric, train-only, and test-only correction on PatchCore.
Symmetric application achieves the highest mean on both shift-prone benchmarks, while the ranking is mixed on VisA and RAD.
We use symmetric application to keep the reference and test features in the same transformed space.

\begin{table*}[!ht]
\caption{\textbf{Symmetric (train${+}$test) versus one-sided calibration across all four benchmarks ($k{=}8$).} We compare applying SPARC to training features only, test features only, or both. Dataset values are category macro-averages over five seeds; bold denotes the best result in each column.}
\label{tab:ab5}
\centering
\setlength{\tabcolsep}{2.4pt}
\renewcommand{\arraystretch}{1.05}
\resizebox{\textwidth}{!}{%
\begin{tabular}{l ccc cc cc cc}
\toprule
& \multicolumn{3}{c}{\textbf{MVTec AD~2}} & \multicolumn{2}{c}{\textbf{AeBAD-S}} & \multicolumn{2}{c}{\textbf{VisA}} & \multicolumn{2}{c}{\textbf{RAD}} \\
\cmidrule(lr){2-4}\cmidrule(lr){5-6}\cmidrule(lr){7-8}\cmidrule(lr){9-10}
\textbf{Variant} & \shortstack{Image\\AUROC} & AU-PRO$_{0.05}$ & AU-PRO$_{0.3}$ & \shortstack{Image\\AUROC} & AU-PRO$_{0.3}$ & \shortstack{Image\\AUROC} & AU-PRO$_{0.3}$ & \shortstack{Image\\AUROC} & AU-PRO$_{0.3}$ \\
\midrule
base (no correction) & 68.7 & 19.0 & 42.8 & 64.1 & 85.1 & \textbf{94.3} & \textbf{87.8} & 92.1 & 73.2 \\
$+\,$SPARC (train-side only) & 60.2 & 16.9 & 39.9 & 63.8 & 85.2 & 89.0 & 84.6 & 93.2 & 73.8 \\
$+\,$SPARC (test-side only) & 82.9 & 18.6 & 42.6 & 64.6 & 85.6 & 86.8 & 83.3 & \textbf{96.4} & 77.3 \\
$+\,$SPARC (symmetric, ours) & \textbf{83.6} & \textbf{20.4} & \textbf{45.7} & \textbf{66.3} & \textbf{86.3} & 93.9 & 86.5 & 96.2 & \textbf{77.9} \\
\bottomrule
\end{tabular}
}%
\end{table*}

\section{Robustness and Sensitivity}
\label{sec:supp-robustness}
\subsection{Calibration Size}
\label{sec:supp-k-sweep}

\Cref{tab:supp-k-sweep} expands the experiment to the $k \in \{2, 4, 8\}$ sweep. 
On the shift-prone benchmarks, gains often increase with $k$, although the increments vary across detectors and metrics.
A larger calibration set raises the algebraic maximum rank of the estimated disagreement subspace.
On VisA and RAD, the changes are generally smaller and more mixed, with detector-specific exceptions.
Pooled across VisA and RAD, the Image AUROC $|\Delta| \le 1.8$~pp for every detector except MuSc.

\begin{table*}[!ht]
\caption{\textbf{Calibration-size sensitivity.} Paired changes for $k\in\{2,4,8\}$.}
\label{tab:supp-k-sweep}
\centering
\setlength{\tabcolsep}{1.5pt}
\renewcommand{\arraystretch}{0.96}
\begin{tabular}{@{}l|rrr|rrr|rrr|rrr@{}}
\toprule
& \multicolumn{3}{c|}{\textbf{MVTec AD~2}} & \multicolumn{3}{c|}{\textbf{AeBAD-S}} & \multicolumn{3}{c|}{\textbf{VisA}} & \multicolumn{3}{c}{\textbf{RAD}} \\
\cmidrule(lr){2-4}\cmidrule(lr){5-7}\cmidrule(lr){8-10}\cmidrule(lr){11-13}
\textbf{Method} & $k{=}2$ & $k{=}4$ & $k{=}8$ & $k{=}2$ & $k{=}4$ & $k{=}8$ & $k{=}2$ & $k{=}4$ & $k{=}8$ & $k{=}2$ & $k{=}4$ & $k{=}8$ \\
\midrule
\multicolumn{13}{@{}l}{\textit{Image AUROC $\Delta$}} \\
SubspaceAD & +4.9 & +9.0 & +14.9 & +8.2 & +11.7 & +14.5 & +0.0 & +0.0 & +0.1 & +0.1 & +0.2 & +0.1 \\
AnomalyDINO & +1.8 & +4.5 & +10.3 & +0.5 & +0.8 & +1.8 & -0.2 & -0.7 & -1.6 & -0.1 & -0.2 & -0.2 \\
PaDiM & +5.9 & +11.6 & +21.7 & +0.2 & +0.4 & +1.0 & +0.0 & +0.2 & +0.3 & +0.1 & +0.2 & +0.3 \\
PatchCore & +3.5 & +7.8 & +14.8 & +2.4 & +2.2 & +2.2 & -0.3 & -0.4 & -0.3 & +1.2 & +2.7 & +4.0 \\
MuSc & +5.9 & +17.3 & +33.1 & -0.1 & +2.3 & +8.2 & -1.4 & -1.5 & -1.4 & -1.7 & -5.9 & -3.2 \\
WinCLIP & +3.7 & +9.1 & +16.4 & -0.4 & -0.5 & +0.6 & +0.5 & +0.9 & +1.1 & +0.1 & +1.7 & +3.0 \\
AnomalyCLIP$_{\mathrm{V}}$ & +0.0 & +0.0 & +0.0 & +0.0 & +0.0 & +0.0 & +0.0 & +0.0 & +0.0 & +0.0 & +0.0 & +0.0 \\
AnomalyCLIP$_{\mathrm{M}}$ & +0.0 & +0.0 & +0.0 & +0.0 & +0.0 & +0.0 & +0.0 & +0.0 & +0.0 & +0.0 & +0.0 & +0.0 \\
SPADE & +3.6 & +7.2 & +17.2 & +1.0 & +2.2 & +5.2 & -0.1 & +0.2 & +0.5 & +0.4 & +2.7 & +5.8 \\
\midrule
\multicolumn{13}{@{}l}{\textit{AU-PRO$_{0.3}$ $\Delta$}} \\
SubspaceAD & +2.9 & +5.1 & +8.0 & +2.4 & +3.0 & +3.4 & -0.0 & -0.1 & -0.2 & +0.2 & +0.3 & +0.6 \\
AnomalyDINO & -0.3 & +1.0 & +3.8 & -0.1 & +1.3 & +2.8 & -0.3 & -0.4 & -0.4 & -0.5 & +0.3 & +1.1 \\
PaDiM & +2.1 & +3.6 & +6.5 & +0.1 & +0.3 & +0.6 & +0.0 & +0.1 & +0.1 & +0.2 & +0.4 & +0.9 \\
PatchCore & +0.1 & +1.1 & +2.9 & +0.2 & +0.6 & +1.2 & -0.6 & -0.8 & -1.3 & +0.3 & +2.3 & +4.7 \\
MuSc & +0.6 & +1.7 & +3.7 & -1.5 & +0.1 & +0.7 & -1.5 & -2.0 & -2.7 & -1.5 & -2.8 & -3.3 \\
WinCLIP & +0.3 & +0.9 & +2.1 & -0.3 & -0.2 & -0.0 & +0.2 & +0.3 & +0.8 & -0.7 & -0.2 & +1.1 \\
AnomalyCLIP$_{\mathrm{V}}$ & +0.7 & +0.5 & -0.1 & +0.0 & -0.6 & -1.8 & -0.0 & -0.1 & -0.2 & +1.0 & +0.4 & +0.0 \\
AnomalyCLIP$_{\mathrm{M}}$ & +0.9 & +1.1 & +1.2 & -0.0 & -0.2 & -0.7 & +0.1 & +0.1 & +0.1 & +0.2 & +0.3 & +0.6 \\
SPADE & +1.0 & +2.4 & +4.6 & +0.2 & +0.3 & +0.5 & -0.5 & -1.2 & -1.9 & +0.4 & +0.0 & -0.6 \\
\midrule
\multicolumn{13}{@{}l}{\textit{AU-PRO$_{0.05}$ $\Delta$}} \\
SubspaceAD & +3.6 & +6.5 & +10.2 & \multicolumn{3}{c|}{--} & \multicolumn{3}{c|}{--} & \multicolumn{3}{c}{--} \\
AnomalyDINO & -0.0 & +0.5 & +1.8 & \multicolumn{3}{c|}{--} & \multicolumn{3}{c|}{--} & \multicolumn{3}{c}{--} \\
PaDiM & +1.0 & +2.4 & +5.8 & \multicolumn{3}{c|}{--} & \multicolumn{3}{c|}{--} & \multicolumn{3}{c}{--} \\
PatchCore & -0.4 & +0.2 & +1.3 & \multicolumn{3}{c|}{--} & \multicolumn{3}{c|}{--} & \multicolumn{3}{c}{--} \\
MuSc & +0.0 & +1.2 & +3.2 & \multicolumn{3}{c|}{--} & \multicolumn{3}{c|}{--} & \multicolumn{3}{c}{--} \\
WinCLIP & +0.3 & +0.3 & +0.7 & \multicolumn{3}{c|}{--} & \multicolumn{3}{c|}{--} & \multicolumn{3}{c}{--} \\
AnomalyCLIP$_{\mathrm{V}}$ & +0.3 & +0.5 & +0.3 & \multicolumn{3}{c|}{--} & \multicolumn{3}{c|}{--} & \multicolumn{3}{c}{--} \\
AnomalyCLIP$_{\mathrm{M}}$ & +0.6 & +0.7 & +0.4 & \multicolumn{3}{c|}{--} & \multicolumn{3}{c|}{--} & \multicolumn{3}{c}{--} \\
SPADE & +0.4 & +0.3 & +1.1 & \multicolumn{3}{c|}{--} & \multicolumn{3}{c|}{--} & \multicolumn{3}{c}{--} \\
\bottomrule
\end{tabular}
\end{table*}

\subsection{Calibration-Set Contamination}
\label{sec:supp-ab4}

A practical concern is that normality verification is imperfect so that anomalous images may enter the calibration set.
\Cref{tab:supp-ab4} evaluates this robustness by replacing verified normals in the $k{=}8$ calibration set with same-category anomalies and refitting SPARC.
\begin{table}[ht!]

\centering
\captionof{table}{\textbf{Calibration contamination.} Paired changes on MVTec AD~2 as normal calibration images are replaced by same-category anomalies.}
\label{tab:supp-ab4}
\centering
\setlength{\tabcolsep}{2.5pt}
\renewcommand{\arraystretch}{0.92}
\begin{tabular*}{\columnwidth}{@{\extracolsep{\fill}}lrrrrr@{}}
\toprule
\textbf{Method} & 0/8 & 1/8 & 2/8 & 3/8 & 4/8 \\
\midrule
\multicolumn{6}{@{}l}{\textit{Image AUROC $\Delta$}} \\
SubspaceAD & +14.9 & +11.5 & +8.5 & +3.5 & +0.5 \\
AnomalyDINO & +10.3 & +7.3 & +4.2 & +0.1 & -3.0 \\
PaDiM & +21.7 & +18.7 & +13.7 & +7.5 & +3.4 \\
PatchCore & +14.8 & +11.7 & +8.0 & +2.6 & -1.7 \\
MuSc & +33.1 & +27.9 & +22.0 & +13.3 & +6.2 \\
WinCLIP & +16.4 & +13.4 & +10.3 & +6.0 & +1.5 \\
AnomalyCLIP$_{\mathrm{V}}$ & +0.0 & +0.0 & +0.0 & +0.0 & +0.0 \\
AnomalyCLIP$_{\mathrm{M}}$ & +0.0 & +0.0 & +0.0 & +0.0 & +0.0 \\
SPADE & +17.2 & +14.6 & +12.8 & +6.4 & +4.9 \\
\midrule
\multicolumn{6}{@{}l}{\textit{AU-PRO$_{0.3}$ $\Delta$}} \\
SubspaceAD & +8.0 & +5.1 & +2.6 & +1.1 & -2.5 \\
AnomalyDINO & +3.8 & +1.3 & -1.7 & -2.9 & -5.1 \\
PaDiM & +6.5 & +4.8 & +2.8 & +1.5 & -1.1 \\
PatchCore & +2.9 & +1.2 & -1.5 & -2.1 & -4.2 \\
MuSc & +3.7 & +1.9 & -0.1 & -2.0 & -5.3 \\
WinCLIP & +2.1 & +1.3 & +0.9 & +0.5 & -1.5 \\
AnomalyCLIP$_{\mathrm{V}}$ & -0.1 & -0.4 & -3.1 & -3.8 & -4.7 \\
AnomalyCLIP$_{\mathrm{M}}$ & +1.2 & +0.7 & -1.7 & -3.3 & -5.2 \\
SPADE & +4.6 & +3.2 & +1.6 & +0.2 & -2.1 \\
\midrule
\multicolumn{6}{@{}l}{\textit{AU-PRO$_{0.05}$ $\Delta$}} \\
SubspaceAD & +10.2 & +7.8 & +5.5 & +4.2 & +2.4 \\
AnomalyDINO & +1.8 & -0.1 & -2.0 & -3.6 & -4.5 \\
PaDiM & +5.8 & +5.5 & +4.0 & +1.4 & +3.0 \\
PatchCore & +1.3 & +1.5 & -0.1 & -1.1 & -1.9 \\
MuSc & +3.2 & +2.1 & +0.8 & -0.4 & -2.5 \\
WinCLIP & +0.7 & -0.1 & -0.4 & -0.5 & -2.0 \\
AnomalyCLIP$_{\mathrm{V}}$ & +0.3 & -0.3 & -2.4 & -2.4 & -3.5 \\
AnomalyCLIP$_{\mathrm{M}}$ & +0.4 & -0.5 & -2.2 & -2.9 & -3.6 \\
SPADE & +1.1 & +0.2 & -0.1 & -1.4 & -2.5 \\
\bottomrule
\end{tabular*}

\end{table}

All seven detectors whose image scores are affected by SPARC retain positive Image AUROC gains through $3/8$ contamination.
The pooled Image AUROC gain is $+4.4$~pp at $3/8$ and remains positive at $+1.3$~pp at $4/8$, although detector-specific declines emerge.
Pooled AU-PRO$_{0.3}$ degrades earlier and reaches $-3.5$~pp at $4/8$, indicating greater sensitivity of localization to calibration contamination.

This suggests a potential label-free deployment mode in which the first few images from a low-defect-rate lot are used without verifying each image as normal, provided that calibration batches are unlikely to contain $50\%$ or more anomalous images.
The present evidence for this mode is limited to image-level detection as localization degrades at lower contamination levels. 
Calibration from an unfiltered deployment stream remains to be evaluated.

\subsection{Backbone}
\label{sec:supp-ab3}
A central design claim of SPARC is that it operates on per-patch feature representations regardless of how those representations were obtained. 
\Cref{tab:supp-ab3} tests this empirically by pairing SPARC with PatchCore using eight different feature encoders spanning four families: ResNet variants (ResNet-50, WideResNet-50); modern CNN / hybrid (ConvNeXt-B, Swin-B); self-supervised vision transformers (DINOv2-B, DINOv3-B, MAE-B); vision-language (CLIP-B). 
All eight backbones are off-the-shelf pretrained checkpoints from their respective releases.
No fine-tuning is performed.

On MVTec AD~2, every evaluated backbone yields a positive Image AUROC change, ranging from $+13.2$~pp for CLIP-B to $+20.4$~pp for MAE-B.
Results on the other benchmarks are more heterogeneous, including large negative Image AUROC changes for CLIP-B on VisA and RAD and for Swin-B on RAD.
Taken together, these results support SPARC's architectural portability across diverse pretrained feature representations, with consistent gains under the targeted MVTec AD~2 shift and mixed effects elsewhere.

\subsection{VisA Input-Geometry Sensitivity}
\label{sec:supp-visa-geometry}
The main VisA experiments use the resize-and-center-crop preprocessing inherited from the reference pipelines of PatchCore, PaDiM, and SPADE.
Since most VisA images are non-square, we additionally evaluate an aspect-ratio-preserving full-frame variant with padding.
We hold the official one-class split, detector configurations, seed-specific calibration draws, and held-out image identities fixed, changing only the input geometry.

\Cref{tab:visa-fullframe-comparison} shows that input geometry can affect absolute detector performance. 
Still, the sign of the paired SPARC-minus-base change is unchanged for all eight method--metric pairs.
Thus, the qualitative conclusion on VisA does not depend on the center-crop choice, although individual effect magnitudes can change.

\begin{table*}[!ht]
\caption{\textbf{Backbone sensitivity.} PatchCore results with eight off-the-shelf feature encoders at $k{=}8$.}
\label{tab:supp-ab3}
\centering
\setlength{\tabcolsep}{1.5pt}
\renewcommand{\arraystretch}{0.92}
\begin{tabular}{@{}l|rrr|rrr|rrr|rrr@{}}
\toprule
& \multicolumn{3}{c|}{\textbf{MVTec AD~2}} & \multicolumn{3}{c|}{\textbf{AeBAD-S}} & \multicolumn{3}{c|}{\textbf{VisA}} & \multicolumn{3}{c}{\textbf{RAD}} \\
\cmidrule(lr){2-4}\cmidrule(lr){5-7}\cmidrule(lr){8-10}\cmidrule(lr){11-13}
\textbf{Backbone} & base & +S & $\Delta$ & base & +S & $\Delta$ & base & +S & $\Delta$ & base & +S & $\Delta$ \\
\midrule
\multicolumn{13}{@{}l}{\textit{Image AUROC}} \\
clip\_b & 64.5 & \textbf{77.7} & +13.2 & \textbf{61.5} & 60.7 & -0.8 & \textbf{90.2} & 84.1 & -6.0 & \textbf{90.5} & 74.2 & -16.3 \\
convnext\_b & 61.8 & \textbf{78.3} & +16.6 & \textbf{68.3} & 67.0 & -1.4 & \textbf{83.6} & 83.5 & -0.1 & 96.6 & \textbf{96.7} & +0.1 \\
dinov2\_b & 66.6 & \textbf{82.7} & +16.2 & 61.5 & \textbf{67.1} & +5.6 & \textbf{94.5} & 92.7 & -1.8 & 95.6 & \textbf{96.1} & +0.5 \\
dinov3\_b & 65.2 & \textbf{83.0} & +17.8 & 76.6 & 76.6 & -0.0 & \textbf{95.4} & 94.1 & -1.3 & 96.5 & \textbf{96.9} & +0.3 \\
mae\_b & 68.0 & \textbf{88.4} & +20.4 & \textbf{56.7} & 54.9 & -1.8 & \textbf{87.3} & 86.8 & -0.5 & \textbf{99.1} & 98.3 & -0.8 \\
resnet50 & 68.6 & \textbf{85.9} & +17.3 & 65.5 & \textbf{69.3} & +3.7 & \textbf{91.7} & 91.3 & -0.4 & 97.4 & \textbf{98.0} & +0.5 \\
swin\_b & 53.7 & \textbf{71.6} & +17.9 & 43.3 & \textbf{47.4} & +4.1 & 54.5 & \textbf{56.8} & +2.3 & \textbf{63.5} & 57.9 & -5.7 \\
wrn50 & 68.4 & \textbf{86.4} & +18.1 & \textbf{64.8} & 63.8 & -1.1 & \textbf{91.8} & 91.6 & -0.2 & 96.3 & \textbf{97.1} & +0.8 \\
\midrule
\multicolumn{13}{@{}l}{\textit{AU-PRO$_{0.3}$}} \\
clip\_b & 42.4 & \textbf{43.2} & +0.9 & \textbf{73.6} & 71.1 & -2.6 & \textbf{79.8} & 75.9 & -3.8 & \textbf{88.2} & 83.8 & -4.4 \\
convnext\_b & 37.8 & \textbf{41.0} & +3.3 & \textbf{83.6} & 82.6 & -1.0 & 73.9 & \textbf{76.5} & +2.6 & 81.5 & \textbf{83.1} & +1.6 \\
dinov2\_b & 44.0 & \textbf{48.1} & +4.1 & 72.1 & \textbf{76.2} & +4.2 & \textbf{84.4} & 82.5 & -1.8 & 92.4 & \textbf{92.8} & +0.4 \\
dinov3\_b & 45.8 & \textbf{49.4} & +3.6 & 83.5 & \textbf{85.0} & +1.5 & \textbf{85.6} & 83.9 & -1.7 & 92.7 & \textbf{92.9} & +0.2 \\
mae\_b & 40.8 & \textbf{45.8} & +5.0 & 61.9 & \textbf{73.3} & +11.5 & 70.0 & \textbf{70.3} & +0.3 & \textbf{92.5} & 92.3 & -0.2 \\
resnet50 & 39.5 & \textbf{42.6} & +3.1 & 77.1 & \textbf{80.8} & +3.8 & 81.7 & \textbf{82.1} & +0.4 & 88.7 & \textbf{89.3} & +0.5 \\
swin\_b & 14.3 & \textbf{20.2} & +5.9 & 7.9 & \textbf{16.0} & +8.1 & 25.1 & \textbf{36.7} & +11.6 & 13.6 & \textbf{14.3} & +0.7 \\
wrn50 & 42.8 & \textbf{45.9} & +3.1 & \textbf{81.5} & 79.7 & -1.8 & \textbf{82.9} & 82.5 & -0.3 & 87.2 & \textbf{89.6} & +2.4 \\
\midrule
\multicolumn{13}{@{}l}{\textit{AU-PRO$_{0.05}$}} \\
clip\_b & \textbf{21.6} & 20.1 & -1.6 & \multicolumn{3}{c|}{--} & \multicolumn{3}{c|}{--} & \multicolumn{3}{c}{--} \\
convnext\_b & 15.6 & \textbf{17.7} & +2.1 & \multicolumn{3}{c|}{--} & \multicolumn{3}{c|}{--} & \multicolumn{3}{c}{--} \\
dinov2\_b & 20.8 & \textbf{23.1} & +2.4 & \multicolumn{3}{c|}{--} & \multicolumn{3}{c|}{--} & \multicolumn{3}{c}{--} \\
dinov3\_b & 23.0 & \textbf{26.7} & +3.7 & \multicolumn{3}{c|}{--} & \multicolumn{3}{c|}{--} & \multicolumn{3}{c}{--} \\
mae\_b & 17.4 & \textbf{18.6} & +1.2 & \multicolumn{3}{c|}{--} & \multicolumn{3}{c|}{--} & \multicolumn{3}{c}{--} \\
resnet50 & 17.6 & \textbf{19.4} & +1.8 & \multicolumn{3}{c|}{--} & \multicolumn{3}{c|}{--} & \multicolumn{3}{c}{--} \\
swin\_b & 2.1 & \textbf{3.0} & +0.9 & \multicolumn{3}{c|}{--} & \multicolumn{3}{c|}{--} & \multicolumn{3}{c}{--} \\
wrn50 & 21.1 & \textbf{22.0} & +0.9 & \multicolumn{3}{c|}{--} & \multicolumn{3}{c|}{--} & \multicolumn{3}{c}{--} \\
\bottomrule
\end{tabular}
\end{table*}

\begin{table*}[!ht]
\caption{\textbf{VisA input-geometry sensitivity.} Center crop (main) versus aspect-ratio-preserving full-frame resizing with padding (control) at $k{=}8$. Split, settings, calibration draws, and held-out images are fixed. Category macro-averages are mean$\pm$sample SD over five paired seeds. $\Delta$ is SPARC minus base before rounding (pp).}
\label{tab:visa-fullframe-comparison}
\centering
\setlength{\tabcolsep}{2.5pt}
\renewcommand{\arraystretch}{0.92}
\begin{tabular}{@{}l|rrr|rrr@{}}
\toprule
& \multicolumn{3}{c|}{\textbf{Center crop (main)}} & \multicolumn{3}{c}{\textbf{Full frame (control)}} \\
\cmidrule(lr){2-4}\cmidrule(lr){5-7}
\textbf{Method} & base & +S & $\Delta$ & base & +S & $\Delta$ \\
\midrule
\multicolumn{7}{@{}l}{\textit{Image AUROC}} \\
AnomalyDINO & \textbf{$95.3\mathbin{\pm}0.1$} & $93.7\mathbin{\pm}0.2$ & $-1.6\mathbin{\pm}0.1$ & \textbf{$95.4\mathbin{\pm}0.1$} & $93.7\mathbin{\pm}0.2$ & $-1.7\mathbin{\pm}0.1$ \\
PaDiM & $89.3\mathbin{\pm}0.1$ & \textbf{$89.6\mathbin{\pm}0.2$} & $+0.3\mathbin{\pm}0.1$ & $86.3\mathbin{\pm}0.3$ & \textbf{$86.5\mathbin{\pm}0.3$} & $+0.2\mathbin{\pm}0.0$ \\
PatchCore & \textbf{$94.3\mathbin{\pm}0.2$} & $93.9\mathbin{\pm}0.2$ & $-0.3\mathbin{\pm}0.3$ & \textbf{$93.7\mathbin{\pm}0.1$} & $93.0\mathbin{\pm}0.2$ & $-0.7\mathbin{\pm}0.3$ \\
SPADE & $83.7\mathbin{\pm}0.2$ & \textbf{$84.2\mathbin{\pm}0.5$} & $+0.5\mathbin{\pm}0.6$ & $83.4\mathbin{\pm}0.3$ & \textbf{$84.1\mathbin{\pm}0.5$} & $+0.7\mathbin{\pm}0.4$ \\
\midrule
\multicolumn{7}{@{}l}{\textit{AU-PRO$_{0.3}$}} \\
AnomalyDINO & \textbf{$92.4\mathbin{\pm}0.0$} & $92.0\mathbin{\pm}0.2$ & $-0.4\mathbin{\pm}0.2$ & \textbf{$92.4\mathbin{\pm}0.0$} & $92.1\mathbin{\pm}0.2$ & $-0.3\mathbin{\pm}0.2$ \\
PaDiM & $82.8\mathbin{\pm}0.0$ & \textbf{$82.9\mathbin{\pm}0.1$} & $+0.1\mathbin{\pm}0.1$ & $81.8\mathbin{\pm}0.1$ & \textbf{$82.0\mathbin{\pm}0.1$} & $+0.2\mathbin{\pm}0.0$ \\
PatchCore & \textbf{$87.8\mathbin{\pm}0.0$} & $86.5\mathbin{\pm}0.3$ & $-1.3\mathbin{\pm}0.3$ & \textbf{$88.5\mathbin{\pm}0.0$} & $87.5\mathbin{\pm}0.3$ & $-1.0\mathbin{\pm}0.3$ \\
SPADE & \textbf{$87.9\mathbin{\pm}0.0$} & $86.0\mathbin{\pm}0.1$ & $-1.9\mathbin{\pm}0.1$ & \textbf{$91.3\mathbin{\pm}0.0$} & $88.3\mathbin{\pm}0.1$ & $-3.1\mathbin{\pm}0.1$ \\
\bottomrule
\end{tabular}
\end{table*}

\section{Cross-Scene Transfer}
\label{sec:supp-scene-disjoint}

MVTec AD~2 can contain multiple lighting captures of the same physical scene.
The main protocol matches the intended deployment-calibration setting.
A few verified-normal images from the incoming lot calibrate scoring of the remainder of that lot.
Calibration and evaluation may therefore contain different lighting captures associated with the same physical scene.

We additionally evaluate cross-scene transfer by removing all captures associated with calibration-selected scenes from evaluation.
The resulting evaluation set contains only physical scenes not observed during calibration.
This complementary protocol tests whether a correction estimated from a few observed deployment scenes transfers to unseen scenes within the same shifted domain.

\Cref{tab:strict-scene-group} shows that, under cross-scene transfer, Image AUROC gains remain positive for six of the seven affected detectors, although they are smaller than under direct deployment-lot calibration.
Localization results are mixed, with clear gains for PaDiM, AnomalyDINO, and SubspaceAD and weaker or negative changes for WinCLIP and MuSc.

\FloatBarrier
\begin{table*}[!ht]
\caption{\textbf{Cross-scene transfer on MVTec AD~2.} Calibration-scene captures are excluded; values are category macro-averages reported as mean$\pm$sample SD over five paired seeds.}
\label{tab:strict-scene-group}
\centering
\setlength{\tabcolsep}{1.5pt}
\renewcommand{\arraystretch}{1.0}
\resizebox{\textwidth}{!}{%
\begin{tabular}{@{}ll|rrr|rrr|rrr@{}}
\toprule
& & \multicolumn{3}{c|}{\textbf{Image AUROC}} & \multicolumn{3}{c|}{\textbf{AU-PRO$_{0.05}$}} & \multicolumn{3}{c}{\textbf{AU-PRO$_{0.3}$}} \\
\cmidrule(lr){3-5}\cmidrule(lr){6-8}\cmidrule(lr){9-11}
\textbf{Backbone} & \textbf{Method} & base & +S & paired $\Delta$ & base & +S & paired $\Delta$ & base & +S & paired $\Delta$ \\
\midrule
WRN-50 & PatchCore \citep{roth2022patchcore} & 68.3$\pm$2.2 & \textbf{68.9$\pm$2.2} & +0.6$\pm$2.1 & 18.9$\pm$0.1 & \textbf{20.6$\pm$0.3} & +1.7$\pm$0.2 & 42.6$\pm$0.3 & \textbf{43.0$\pm$0.2} & +0.4$\pm$0.2 \\
WRN-50 & PaDiM \citep{defard2021padim} & 60.0$\pm$2.9 & \textbf{62.0$\pm$2.7} & +2.0$\pm$0.6 & 11.3$\pm$0.1 & \textbf{18.4$\pm$0.1} & +7.0$\pm$0.1 & 40.9$\pm$0.2 & \textbf{44.7$\pm$0.3} & +3.8$\pm$0.2 \\
WRN-50 & SPADE \citep{cohen2020spade} & 54.3$\pm$1.0 & \textbf{57.3$\pm$2.1} & +3.1$\pm$1.3 & \textbf{18.3$\pm$0.0} & 18.2$\pm$0.3 & $-$0.1$\pm$0.3 & 41.9$\pm$0.3 & \textbf{42.7$\pm$0.3} & +0.8$\pm$0.2 \\
DINOv2-S/14 & AnomalyDINO \citep{damm2025anomalydino} & 71.5$\pm$0.4 & \textbf{72.2$\pm$0.8} & +0.7$\pm$0.5 & 36.8$\pm$0.0 & \textbf{38.0$\pm$0.2} & +1.2$\pm$0.2 & 59.9$\pm$0.1 & \textbf{61.7$\pm$0.7} & +1.9$\pm$0.7 \\
DINOv2-G & SubspaceAD \citep{lendering2026subspacead} & 71.9$\pm$0.9 & \textbf{77.3$\pm$1.8} & +5.4$\pm$1.0 & 35.3$\pm$0.1 & \textbf{44.4$\pm$1.1} & +9.0$\pm$1.0 & 61.2$\pm$0.1 & \textbf{67.9$\pm$0.7} & +6.7$\pm$0.6 \\
CLIP-B/16 & WinCLIP \citep{jeong2023winclip} & \textbf{62.6$\pm$1.5} & 61.9$\pm$1.8 & $-$0.7$\pm$1.0 & 16.9$\pm$0.1 & \textbf{17.0$\pm$0.5} & +0.1$\pm$0.5 & \textbf{44.9$\pm$0.3} & 44.8$\pm$0.3 & $-$0.1$\pm$0.2 \\
CLIP-L/14 & MuSc \citep{li2024musc} & 63.0$\pm$1.3 & \textbf{63.3$\pm$1.4} & +0.3$\pm$0.7 & \textbf{28.1$\pm$0.5} & 26.0$\pm$0.3 & $-$2.0$\pm$0.4 & \textbf{54.1$\pm$0.2} & 51.8$\pm$0.4 & $-$2.2$\pm$0.4 \\
CLIP-L/14 & AnomalyCLIP$_{\mathrm{V}}$ \citep{zhou2024anomalyclip} & 52.1$\pm$1.0 & 52.1$\pm$1.0 & 0.0$\pm$0.0 & 25.1$\pm$0.1 & \textbf{25.3$\pm$0.7} & +0.2$\pm$0.7 & \textbf{45.1$\pm$0.3} & 44.9$\pm$0.7 & $-$0.2$\pm$0.8 \\
CLIP-L/14 & AnomalyCLIP$_{\mathrm{M}}$ \citep{zhou2024anomalyclip} & 56.6$\pm$0.8 & 56.6$\pm$0.8 & 0.0$\pm$0.0 & \textbf{25.0$\pm$0.1} & 24.9$\pm$0.2 & $-$0.1$\pm$0.2 & 42.5$\pm$0.2 & \textbf{42.6$\pm$0.4} & 0.0$\pm$0.6 \\
\bottomrule
\end{tabular}
}%
\end{table*}

\section{Further Analysis}
\label{sec:supp-mechanism}
\subsection{Removal Mode}
\label{sec:supp-mean-vs-subspace}
\Cref{tab:correction-mode-mean} extends the correction-mode comparison with a per-cell\,$+$\,mean variant, and \Cref{tab:dc-fraction} decomposes the removed disagreement energy. 
Two observations characterize the decision to omit mean removal.

First, the DC component accounts for $12$--$26\%$ of the combined DC and subspace energy across benchmarks, while the centered subspace accounts for the remainder (\Cref{tab:dc-fraction}).

Second, subtracting the mean is not merely insufficient but actively destructive for detectors that score by within-image geometry.
Adding centroid subtraction on top of SPARC collapses AnomalyDINO ($-34.5$~pp) and MuSc ($-55.9$~pp Image AUROC; \Cref{tab:correction-mode-mean}).
On Image AUROC, per-cell subspace projection is the only variant that improves every corrected detector.
On AU-PRO$_{0.3}$, the gains are broad but not universal.

\begin{table*}[!ht]
\caption{\textbf{Mean-removal variants.} Detector-level results for subspace and mean corrections at $k{=}8$.}
\label{tab:correction-mode-mean}
\centering
\setlength{\tabcolsep}{1.5pt}
\renewcommand{\arraystretch}{1.0}
\begin{tabular*}{\textwidth}{@{\extracolsep{\fill}}ll cc cc cc cc cc@{}}
\toprule
& & \multicolumn{2}{c}{\textbf{base}} & \multicolumn{2}{c}{\textbf{mean-only}} & \multicolumn{2}{c}{\textbf{\shortstack{global\\subspace}}} & \multicolumn{2}{c}{\textbf{\shortstack{per-cell SPARC\\(ours)}}} & \multicolumn{2}{c}{\textbf{\shortstack{SPARC\\+ mean}}} \\
\cmidrule(lr){3-4}\cmidrule(lr){5-6}\cmidrule(lr){7-8}\cmidrule(lr){9-10}\cmidrule(lr){11-12}
\textbf{Detector} & \textbf{Metric} & Score & $\Delta$ & Score & $\Delta$ & Score & $\Delta$ & Score & $\Delta$ & Score & $\Delta$ \\
\midrule
SubspaceAD & Image AUROC & 84.7 & +0.0 & 91.5 & +6.8 & 87.2 & +2.5 & 92.1 & +7.4 & \textbf{92.2} & +7.5 \\
 & AU-PRO$_{0.3}$ & 85.2 & +0.0 & 87.6 & +2.5 & 86.4 & +1.2 & 88.1 & +3.0 & \textbf{88.2} & +3.0 \\
 & AU-PRO$_{0.05}$ & 35.4 & +0.0 & 44.5 & +9.1 & 39.0 & +3.6 & 45.6 & +10.2 & \textbf{45.7} & +10.2 \\
\addlinespace[1.5pt]
AnomalyDINO & Image AUROC & 84.0 & +0.0 & 64.7 & -19.3 & 82.9 & -1.1 & \textbf{86.5} & +2.6 & 52.0 & -31.9 \\
 & AU-PRO$_{0.3}$ & 81.4 & +0.0 & 35.5 & -45.9 & 76.9 & -4.5 & \textbf{83.2} & +1.8 & 14.9 & -66.5 \\
 & AU-PRO$_{0.05}$ & 37.0 & +0.0 & 6.3 & -30.6 & 33.0 & -3.9 & \textbf{38.8} & +1.8 & 0.3 & -36.7 \\
\addlinespace[1.5pt]
PaDiM & Image AUROC & 77.1 & +0.0 & 77.1 & +0.0 & 78.0 & +0.9 & \textbf{82.9} & +5.8 & \textbf{82.9} & +5.8 \\
 & AU-PRO$_{0.3}$ & 70.3 & +0.0 & 70.3 & -0.0 & 71.7 & +1.3 & \textbf{72.3} & +2.0 & \textbf{72.3} & +2.0 \\
 & AU-PRO$_{0.05}$ & 11.3 & +0.0 & 11.3 & +0.0 & 15.2 & +3.9 & \textbf{17.1} & +5.8 & \textbf{17.1} & +5.8 \\
\addlinespace[1.5pt]
PatchCore & Image AUROC & 79.8 & +0.0 & 84.0 & +4.2 & 82.4 & +2.6 & \textbf{85.0} & +5.2 & 84.9 & +5.1 \\
 & AU-PRO$_{0.3}$ & 72.2 & +0.0 & 72.9 & +0.6 & \textbf{74.4} & +2.2 & 74.1 & +1.9 & 74.2 & +2.0 \\
 & AU-PRO$_{0.05}$ & 19.0 & +0.0 & 17.2 & -1.9 & 20.3 & +1.2 & \textbf{20.4} & +1.3 & 19.8 & +0.7 \\
\addlinespace[1.5pt]
MuSc & Image AUROC & 59.2 & +0.0 & 21.2 & -38.0 & 60.2 & +1.0 & \textbf{68.3} & +9.2 & 12.4 & -46.7 \\
 & AU-PRO$_{0.3}$ & \textbf{77.9} & +0.0 & 24.2 & -53.7 & 76.5 & -1.4 & 77.5 & -0.4 & 6.4 & -71.5 \\
 & AU-PRO$_{0.05}$ & 26.6 & +0.0 & 2.5 & -24.1 & 27.3 & +0.7 & \textbf{29.9} & +3.2 & 0.0 & -26.6 \\
\addlinespace[1.5pt]
AnomalyCLIP$_{\mathrm{V}}$ & Image AUROC & \textbf{73.3} & +0.0 & \textbf{73.3} & +0.0 & \textbf{73.3} & +0.0 & \textbf{73.3} & +0.0 & \textbf{73.3} & +0.0 \\
 & AU-PRO$_{0.3}$ & \textbf{73.3} & +0.0 & 56.8 & -16.5 & 70.5 & -2.8 & 72.8 & -0.5 & 54.3 & -19.0 \\
 & AU-PRO$_{0.05}$ & 25.2 & +0.0 & 9.0 & -16.3 & 22.5 & -2.7 & \textbf{25.5} & +0.3 & 11.1 & -14.2 \\
\addlinespace[1.5pt]
AnomalyCLIP$_{\mathrm{M}}$ & Image AUROC & \textbf{75.3} & +0.0 & \textbf{75.3} & +0.0 & \textbf{75.3} & +0.0 & \textbf{75.3} & +0.0 & \textbf{75.3} & +0.0 \\
 & AU-PRO$_{0.3}$ & 71.6 & +0.0 & 62.2 & -9.4 & 70.0 & -1.6 & \textbf{71.9} & +0.3 & 58.8 & -12.8 \\
 & AU-PRO$_{0.05}$ & 25.2 & +0.0 & 10.2 & -14.9 & 23.5 & -1.7 & \textbf{25.5} & +0.4 & 11.7 & -13.5 \\
\addlinespace[1.5pt]
SPADE & Image AUROC & 51.2 & +0.0 & 51.2 & +0.0 & 55.3 & +4.1 & \textbf{58.4} & +7.2 & \textbf{58.4} & +7.3 \\
 & AU-PRO$_{0.3}$ & 74.7 & +0.0 & 75.2 & +0.6 & \textbf{75.3} & +0.6 & \textbf{75.3} & +0.6 & 75.1 & +0.5 \\
 & AU-PRO$_{0.05}$ & 18.4 & +0.0 & \textbf{19.5} & +1.0 & 18.6 & +0.2 & \textbf{19.5} & +1.1 & 19.3 & +0.9 \\
\bottomrule
\end{tabular*}
\end{table*}

\subsection{Calibration Source}
\label{sec:supp-trainside}
SPARC calibrates on $k{\le}8$ verified-normal images from the incoming deployment lot. 
To confirm that the deployment origin of these normals drives the correction, we repeat the entire pipeline, drawing the $k$ normals from the training pool instead, holding the host detector, its source training, and the evaluation split fixed. 
\Cref{tab:trainside} reports both sources. 
Train-side calibration is close to inert across all six patch-feature detectors (pooled $-0.3$~pp Image AUROC, with small shifts in either direction), whereas deployment-lot calibration yields the main-result gains (pooled $+14.3$~pp).
Deployment-lot calibration outperforms training-pool calibration for all six detectors on both primary metrics after Holm correction.
The removable disagreement captured by SPARC is therefore specific to the deployment lot and is not exposed by training normals.

\subsection{Anomaly Signal}
\label{sec:supp-at2}

We directly measure two complementary quantities: how anomalous patches align with the calibration-defined subspace and how SPARC changes the feature-space contrast between anomalous and normal patches.

\subsubsection{Anomaly--Subspace Alignment}
At each test cell, let
\begin{equation}
g_{i,j}^{\mathrm{test}}
=
f_{i,j}^{\mathrm{test}}-\bar{\mu}_{i,j}.
\end{equation}
For an anomaly cell, defined as a cell intersecting the ground-truth anomaly mask of an anomalous image, we compute
\begin{equation}
\rho_{\mathrm{anom}}
=
\frac{
    \left\|V_{i,j}V_{i,j}^{\top}g_{i,j}^{\mathrm{test}}\right\|_2^2
}{
    \left\|g_{i,j}^{\mathrm{test}}\right\|_2^2
}.
\end{equation}
This quantity measures the fraction of the centered test feature's energy lying in the calibration-defined subspace
$\widehat{\mathcal{S}}_{i,j}$.
Because $V_{i,j}V_{i,j}^{\top}$ is an orthogonal projector, it is also the fraction of centered feature energy removed by the SPARC correction described in the main paper.

If anomaly directions were isotropic in $\mathbb{R}^{d}$, their expected overlap with a rank-$r$ subspace would be $r/d$.
Here, the diagnostic operates on the concatenated WRN-50 layer2+layer3 feature before detector-specific feature reduction or scoring, giving $d=1536$.
Since $r\leq k-1=7$ at $k{=}8$, the isotropic rank-capacity reference is at most
$7/1536\approx0.0046$ by \Cref{prop:capacity}.

As summarized in \Cref{tab:anomaly-direction-diagnostics}, the cell-pooled values of $\rho_{\mathrm{anom}}$ are $0.119$, $0.064$, $0.195$, and $0.072$ on MVTec AD~2, AeBAD-S, VisA, and RAD, respectively.
These values are $26.2$, $14.1$, $42.9$, and $15.8$ times this conservative upper reference.
Thus, anomaly directions overlap nontrivially with the calibration-defined subspace.

\subsubsection{Anomaly--Normal Contrast}
A complementary question is whether SPARC widens the feature-space separation between anomalous and normal cells.
For each cell, we use the calibration-centered Euclidean distances
\begin{align}
M_{\mathrm{Base}}
&=
\left\|g_{i,j}^{\mathrm{test}}\right\|_2,
\\
M_{\mathrm{SPARC}}
&=
\left\|
    \left(I-V_{i,j}V_{i,j}^{\top}\right)
    g_{i,j}^{\mathrm{test}}
\right\|_2.
\end{align}
We then define
\begin{equation}
\operatorname{contrast}(M)
=
\bar{M}_{\mathrm{anom}}
-
\bar{M}_{\mathrm{norm}},
\end{equation}
where $\bar{M}_{\mathrm{anom}}$ is the mean over ground-truth anomaly cells and $\bar{M}_{\mathrm{norm}}$ is the mean over normal cells.
The normal pool includes all cells from held-out normal images and cells outside the ground-truth anomaly mask in anomalous images.
We use Euclidean distance because estimating a nonsingular per-cell Mahalanobis metric is infeasible when $k\leq8\ll d$.
This is therefore a calibration-centered feature diagnostic rather than a detector-specific anomaly score.

We compute base and SPARC contrast by pooling the underlying cell-level sums and counts over all categories and five calibration seeds within each dataset.
As reported in \Cref{tab:anomaly-direction-diagnostics}, the widening is largest on the shift-prone benchmarks.
MVTec AD~2 increases from $0.749$ to $1.139$
($\Delta=+0.390$, $+52.1\%$ relative), while AeBAD-S increases from $1.576$ to $1.776$
($\Delta=+0.200$, $+12.7\%$).
The change is small on the shift-free benchmarks: VisA increases from $1.338$ to $1.342$
($\Delta=+0.004$, $+0.3\%$), while RAD increases from $1.693$ to $1.822$
($\Delta=+0.130$, $+7.7\%$).
Together, these measurements show that anomaly features overlap with the calibration-defined subspace and that SPARC widens the measured anomaly--normal feature contrast primarily on the shift-prone benchmarks.

\begin{table}[!ht]
\caption{\textbf{Disagreement-energy decomposition.} The DC fraction is the share of the combined DC and subspace energy attributable to the per-cell mean offset.}
\label{tab:dc-fraction}
\centering
\begin{tabular*}{\columnwidth}{@{\extracolsep{\fill}}lrrr@{}}
\toprule
Dataset & DC energy & Subspace energy & DC fraction \\
\midrule
MVTec AD~2 & 43.77 & 154.12 & 0.221 \\
AeBAD-S & 76.84 & 232.90 & 0.248 \\
VisA & 16.67 & 117.31 & 0.124 \\
RAD & 65.69 & 184.73 & 0.262 \\
\bottomrule
\end{tabular*}
\end{table}

\begin{table}[!ht]
\caption{\textbf{Calibration-source control.} One-sided Wilcoxon signed-rank tests compare deployment-lot and training-pool calibration under the same paired protocol. Holm correction is applied jointly across six detectors and two primary metrics.}
\label{tab:trainside}
\centering
\setlength{\tabcolsep}{2.1pt}
\renewcommand{\arraystretch}{1.02}
\begin{tabular*}{\columnwidth}{@{\extracolsep{\fill}}lrrrrr@{}}
\toprule
\textbf{Method} & Deploy $\Delta$ & Train $\Delta$ & Difference & Raw $p$ & Holm $p$ \\
\midrule
\multicolumn{6}{@{}l}{\textit{Image AUROC}} \\
PatchCore & +10.6 & -2.5 & +13.1 & $<$.001$^{***}$ & .003$^{**}$ \\
PaDiM & +14.8 & -0.0 & +14.8 & $<$.001$^{***}$ & .003$^{**}$ \\
SPADE & +13.2 & +1.2 & +12.0 & $<$.001$^{***}$ & .003$^{**}$ \\
AnomalyDINO & +7.5 & -0.3 & +7.7 & .001$^{**}$ & .003$^{**}$ \\
SubspaceAD & +14.8 & -0.2 & +15.0 & $<$.001$^{***}$ & .003$^{**}$ \\
MuSc & +24.8 & +0.2 & +24.6 & $<$.001$^{***}$ & .003$^{**}$ \\
\midrule
\multicolumn{6}{@{}l}{\textit{AU-PRO$_{0.3}$}} \\
PatchCore & +2.4 & -0.2 & +2.5 & .002$^{**}$ & .003$^{**}$ \\
PaDiM & +4.5 & +0.1 & +4.4 & $<$.001$^{***}$ & .003$^{**}$ \\
SPADE & +3.2 & +0.1 & +3.1 & $<$.001$^{***}$ & .003$^{**}$ \\
AnomalyDINO & +3.5 & -0.6 & +4.1 & $<$.001$^{***}$ & .003$^{**}$ \\
SubspaceAD & +6.5 & -0.6 & +7.0 & $<$.001$^{***}$ & .003$^{**}$ \\
MuSc & +2.7 & -1.3 & +4.0 & $<$.001$^{***}$ & .003$^{**}$ \\
\bottomrule
\end{tabular*}
\end{table}

\begin{table*}[!ht]
\caption{\textbf{Anomaly--subspace alignment and anomaly--normal contrast.} $\Delta$ is SPARC minus base. Ratios and changes are computed before rounding; displayed entries are rounded independently. VisA uses the center-crop geometry of the primary experiment; the full-frame control is reported in \Cref{tab:visa-fullframe-comparison}. Feature setup, the $7/1536$ reference, and the contrast definition are given in \Cref{sec:supp-at2}.}
\label{tab:anomaly-direction-diagnostics}
\centering
\setlength{\tabcolsep}{2pt}
\renewcommand{\arraystretch}{1.05}
\begin{tabular}{lrrrrrr}
\toprule
& \multicolumn{2}{c}{Anomaly--subspace alignment} & \multicolumn{4}{c}{Anomaly--normal contrast} \\
\cmidrule(lr){2-3}\cmidrule(lr){4-7}
Dataset & $\rho_{\mathrm{anom}}$ & $\rho_{\mathrm{anom}}/(7/1536)$ & base & SPARC & $\Delta$ & Rel.\ $\Delta$ \\
\midrule
MVTec AD~2 & 0.119 & 26.2$\times$ & 0.749 & 1.139 & +0.390 & +52.1\%  \\
AeBAD-S & 0.064 & 14.1$\times$ & 1.576 & 1.776 & +0.200 & +12.7\%  \\
VisA & 0.195 & 42.9$\times$ & 1.338 & 1.342 & +0.004 & +0.3\%  \\
RAD & 0.072 & 15.8$\times$ & 1.693 & 1.822 & +0.130 & +7.7\%  \\
\bottomrule
\end{tabular}

\end{table*}

\section{Statistical Analysis}
\label{sec:supp-statistics}
\Cref{tab:wilcoxon_holm_main} tests the seven detectors whose image scores depend on corrected patch features on the two primary metrics over the 12 shift-prone categories.
The 14 detector-metric hypotheses form one Holm family.
All seven Image AUROC gains remain significant after correction.
Five AU-PRO$_{0.3}$ gains remain significant, while the MuSc and WinCLIP gains do not.
\Cref{tab:maintable-variability} reports variability across the five calibration draws.

\begin{table*}[!ht]
\caption{\textbf{Primary shift-prone significance tests.} One-sided Wilcoxon signed-rank tests over seed-averaged category changes on MVTec AD~2 and AeBAD-S. Holm correction is applied jointly across the seven detectors and two primary metrics.}
\label{tab:wilcoxon_holm_main}
\centering
\setlength{\tabcolsep}{2.4pt}
\renewcommand{\arraystretch}{0.90}
\begin{tabular*}{\columnwidth}{@{\extracolsep{\fill}}lrrrrr@{}}
\toprule
\textbf{Method} & $N$ & Mean $\Delta$ & Median $\Delta$ & Raw $p$ & Holm $p$ \\
\midrule
\multicolumn{6}{@{}l}{\textit{Image AUROC}} \\
PatchCore & 12 & +10.62 & +6.86 & .001$^{**}$ & .006$^{**}$ \\
PaDiM & 12 & +14.78 & +13.22 & $<$.001$^{***}$ & .003$^{**}$ \\
SPADE & 12 & +13.18 & +11.14 & $<$.001$^{***}$ & .003$^{**}$ \\
AnomalyDINO & 12 & +7.47 & +6.87 & $<$.001$^{***}$ & .003$^{**}$ \\
SubspaceAD & 12 & +14.77 & +13.36 & $<$.001$^{***}$ & .003$^{**}$ \\
WinCLIP & 12 & +11.09 & +9.74 & .001$^{**}$ & .006$^{**}$ \\
MuSc & 12 & +24.81 & +23.28 & $<$.001$^{***}$ & .003$^{**}$ \\
\midrule
\multicolumn{6}{@{}l}{\textit{AU-PRO$_{0.3}$}} \\
PatchCore & 12 & +2.36 & +2.08 & $<$.001$^{***}$ & .003$^{**}$ \\
PaDiM & 12 & +4.55 & +2.58 & $<$.001$^{***}$ & .003$^{**}$ \\
SPADE & 12 & +3.21 & +1.29 & $<$.001$^{***}$ & .003$^{**}$ \\
AnomalyDINO & 12 & +3.46 & +2.76 & .001$^{**}$ & .006$^{**}$ \\
SubspaceAD & 12 & +6.47 & +4.57 & $<$.001$^{***}$ & .003$^{**}$ \\
WinCLIP & 12 & +1.40 & +1.03 & .039$^{*}$ & .064 \\
MuSc & 12 & +2.68 & +2.17 & .032$^{*}$ & .064 \\
\bottomrule
\end{tabular*}
\end{table*}
\begin{table*}[!ht]
\caption{\textbf{Calibration-draw variability.} Mean paired changes and sample standard deviations over five seeds.}
\label{tab:maintable-variability}
\centering
\setlength{\tabcolsep}{2.2pt}
\renewcommand{\arraystretch}{0.88}
\begin{tabular}{@{}l cc cc cc cc@{}}
\toprule
& \multicolumn{2}{c}{\textbf{MVTec AD~2}} & \multicolumn{2}{c}{\textbf{AeBAD-S}} & \multicolumn{2}{c}{\textbf{VisA}} & \multicolumn{2}{c}{\textbf{RAD}} \\
\cmidrule(lr){2-3}\cmidrule(lr){4-5}\cmidrule(lr){6-7}\cmidrule(lr){8-9}
\textbf{Method} & Mean $\Delta$ & Seed SD & Mean $\Delta$ & Seed SD & Mean $\Delta$ & Seed SD & Mean $\Delta$ & Seed SD \\
\midrule
\multicolumn{9}{@{}l}{\textit{Image AUROC}} \\
SubspaceAD & +14.9 & 0.8 & +14.5 & 1.0 & +0.1 & 0.1 & +0.1 & 0.1 \\
AnomalyDINO & +10.3 & 1.4 & +1.8 & 0.4 & -1.6 & 0.1 & -0.2 & 0.1 \\
PaDiM & +21.7 & 2.9 & +1.0 & 0.5 & +0.3 & 0.1 & +0.3 & 0.1 \\
PatchCore & +14.8 & 2.0 & +2.2 & 1.3 & -0.3 & 0.3 & +4.0 & 0.4 \\
MuSc & +33.1 & 3.1 & +8.2 & 1.9 & -1.4 & 0.4 & -3.2 & 1.5 \\
WinCLIP & +16.4 & 1.3 & +0.6 & 1.1 & +1.1 & 0.3 & +3.0 & 0.7 \\
AnomalyCLIP$_{\mathrm{V}}$ & +0.0 & 0.0 & +0.0 & 0.0 & +0.0 & 0.0 & +0.0 & 0.0 \\
AnomalyCLIP$_{\mathrm{M}}$ & +0.0 & 0.0 & +0.0 & 0.0 & +0.0 & 0.0 & +0.0 & 0.0 \\
SPADE & +17.2 & 2.3 & +5.2 & 1.0 & +0.5 & 0.6 & +5.8 & 0.8 \\
\midrule
\multicolumn{9}{@{}l}{\textit{AU-PRO$_{0.3}$}} \\
SubspaceAD & +8.0 & 0.8 & +3.4 & 0.1 & -0.2 & 0.0 & +0.6 & 0.2 \\
AnomalyDINO & +3.8 & 0.4 & +2.8 & 0.4 & -0.4 & 0.2 & +1.1 & 0.1 \\
PaDiM & +6.5 & 1.2 & +0.6 & 0.1 & +0.1 & 0.1 & +0.9 & 0.1 \\
PatchCore & +2.9 & 1.9 & +1.2 & 0.2 & -1.3 & 0.3 & +4.7 & 1.3 \\
MuSc & +3.7 & 1.3 & +0.7 & 0.7 & -2.7 & 0.2 & -3.3 & 0.4 \\
WinCLIP & +2.1 & 0.8 & -0.0 & 0.5 & +0.8 & 0.2 & +1.1 & 0.5 \\
AnomalyCLIP$_{\mathrm{V}}$ & -0.1 & 0.5 & -1.8 & 0.6 & -0.2 & 0.1 & +0.0 & 0.9 \\
AnomalyCLIP$_{\mathrm{M}}$ & +1.2 & 0.3 & -0.7 & 0.1 & +0.1 & 0.0 & +0.6 & 0.5 \\
SPADE & +4.6 & 0.7 & +0.5 & 0.1 & -1.9 & 0.1 & -0.6 & 0.2 \\
\midrule
\multicolumn{9}{@{}l}{\textit{AU-PRO$_{0.05}$}} \\
SubspaceAD & +10.2 & 1.1 & \multicolumn{2}{c}{--} & \multicolumn{2}{c}{--} & \multicolumn{2}{c}{--} \\
AnomalyDINO & +1.8 & 0.5 & \multicolumn{2}{c}{--} & \multicolumn{2}{c}{--} & \multicolumn{2}{c}{--} \\
PaDiM & +5.8 & 3.1 & \multicolumn{2}{c}{--} & \multicolumn{2}{c}{--} & \multicolumn{2}{c}{--} \\
PatchCore & +1.3 & 1.8 & \multicolumn{2}{c}{--} & \multicolumn{2}{c}{--} & \multicolumn{2}{c}{--} \\
MuSc & +3.2 & 0.6 & \multicolumn{2}{c}{--} & \multicolumn{2}{c}{--} & \multicolumn{2}{c}{--} \\
WinCLIP & +0.7 & 0.6 & \multicolumn{2}{c}{--} & \multicolumn{2}{c}{--} & \multicolumn{2}{c}{--} \\
AnomalyCLIP$_{\mathrm{V}}$ & +0.3 & 0.3 & \multicolumn{2}{c}{--} & \multicolumn{2}{c}{--} & \multicolumn{2}{c}{--} \\
AnomalyCLIP$_{\mathrm{M}}$ & +0.4 & 0.3 & \multicolumn{2}{c}{--} & \multicolumn{2}{c}{--} & \multicolumn{2}{c}{--} \\
SPADE & +1.1 & 0.1 & \multicolumn{2}{c}{--} & \multicolumn{2}{c}{--} & \multicolumn{2}{c}{--} \\
\bottomrule
\end{tabular}
\end{table*}

\section{Per-Category Results}
\label{sec:supp-categories}
Per-category Image AUROC and AU-PRO$_{0.3}$ are reported for MVTec AD~2 in \Cref{tab:supp-percat-mvtecad2-auroc,tab:supp-percat-mvtecad2-aupro}, AeBAD-S in \Cref{tab:supp-percat-aebad-auroc,tab:supp-percat-aebad-aupro}, RAD in \Cref{tab:supp-percat-rad-auroc,tab:supp-percat-rad-aupro}, and VisA in \Cref{tab:supp-percat-visa-auroc,tab:supp-percat-visa-aupro}.
Each table reports every detector--category pair at $k{=}8$, averaged over five calibration draws.
Together, they expand the dataset-level means reported in the Experiments section of the main paper into their per-category components.
The baselines do not decompose by detector, so we report them separately: \Cref{tab:supp-baseline-percat} provides the per-category FastRecon and FastRef scores underlying \Cref{tab:comparison-baselines-auroc,tab:comparison-baselines-aupro}.

\section{Qualitative Results}
\label{sec:supp-qualitative}
In \Cref{fig:qual-mvtec2,fig:qual-visa-p1,fig:qual-visa-p2}, we show SPARC-corrected anomaly maps for every detector on one representative test image per category of MVTec AD~2 and VisA.

\FloatBarrier

\begin{table*}[!ht]
\caption{\textbf{Per-category few-shot baselines.} Entries are FastRecon/FastRef at $k{=}8$.}
\label{tab:supp-baseline-percat}
\centering
\setlength{\tabcolsep}{3pt}
\renewcommand{\arraystretch}{0.85}
\begin{tabular}{@{}lrrr@{}}
\toprule
Category & Image AUROC & AU-PRO$_{0.05}$ & AU-PRO$_{0.3}$ \\
\midrule
\multicolumn{4}{@{}l}{\textit{MVTec AD~2}} \\
can & 55.7/64.8 & 4.9/10.5 & 26.3/32.9 \\
fabric & 67.8/73.2 & 3.5/3.4 & 18.3/18.0 \\
fruit\_jelly & 78.1/94.5 & 32.4/35.7 & 60.5/63.5 \\
rice & 75.4/81.8 & 17.9/22.0 & 40.3/43.0 \\
sheet\_metal & 91.2/92.0 & 9.5/9.5 & 26.2/25.5 \\
vial & 81.4/96.0 & 38.3/54.0 & 69.8/85.0 \\
wallplugs & 75.5/77.2 & 9.7/12.9 & 35.4/39.3 \\
walnuts & 78.9/89.8 & 23.1/33.5 & 56.0/63.8 \\
\midrule
\multicolumn{4}{@{}l}{\textit{AeBAD-S}} \\
background & 56.1/69.3 & --/-- & 80.1/84.1 \\
illumination & 62.1/73.9 & --/-- & 78.9/81.1 \\
same & 57.9/66.7 & --/-- & 81.0/83.7 \\
view & 52.2/62.2 & --/-- & 75.2/80.5 \\
\midrule
\multicolumn{4}{@{}l}{\textit{RAD}} \\
bolt & 98.0/97.7 & --/-- & 80.4/85.5 \\
ribbon & 98.3/98.1 & --/-- & 83.2/86.8 \\
sponge & 97.9/98.0 & --/-- & 86.6/91.5 \\
tape & 98.6/98.7 & --/-- & 96.5/97.5 \\
\midrule
\multicolumn{4}{@{}l}{\textit{VisA}} \\
candle & 84.0/92.7 & --/-- & 86.0/94.2 \\
capsules & 52.7/60.7 & --/-- & 37.5/49.8 \\
cashew & 87.3/92.9 & --/-- & 83.1/87.9 \\
chewinggum & 94.6/98.2 & --/-- & 75.3/79.1 \\
fryum & 62.9/89.6 & --/-- & 53.2/76.1 \\
macaroni1 & 69.4/83.4 & --/-- & 80.6/90.0 \\
macaroni2 & 56.9/64.8 & --/-- & 53.7/61.7 \\
pcb1 & 64.2/82.2 & --/-- & 49.2/67.6 \\
pcb2 & 60.9/75.5 & --/-- & 57.9/70.0 \\
pcb3 & 67.0/75.6 & --/-- & 51.9/69.9 \\
pcb4 & 73.3/92.5 & --/-- & 43.8/77.2 \\
pipe\_fryum & 90.7/96.5 & --/-- & 86.8/93.0 \\
\bottomrule
\end{tabular}
\end{table*}
\begin{table*}[!ht]
\caption{\textbf{Per-category Image AUROC on MVTec AD~2.} Results use $k{=}8$.}
\label{tab:supp-percat-mvtecad2-auroc}
\centering
\setlength{\tabcolsep}{2.2pt}
\renewcommand{\arraystretch}{0.90}
\resizebox{\textwidth}{!}{%
\begin{tabular}{llrrrrrrrrr}
\toprule
Category & Variant & PatchCore & PaDiM & SPADE & AnomalyDINO & SubspaceAD & WinCLIP & MuSc & AnomalyCLIP$_{\mathrm{V}}$ & AnomalyCLIP$_{\mathrm{M}}$ \\
\midrule
can & base & 49.7 & 49.7 & 50.0 & 52.5 & 50.0 & 24.2 & 53.0 & 57.2 & 55.2 \\
can & $+\,$SPARC & 60.2 & 58.9 & 63.1 & 63.5 & 68.3 & 44.8 & 69.8 & 57.2 & 55.2 \\
fabric & base & 52.0 & 28.2 & 32.5 & 56.0 & 64.8 & 58.8 & 34.9 & 58.3 & 75.0 \\
fabric & $+\,$SPARC & 73.2 & 65.1 & 63.9 & 70.5 & 77.4 & 86.4 & 71.4 & 58.3 & 75.0 \\
fruit\_jelly & base & 89.2 & 79.2 & 70.8 & 86.5 & 83.4 & 88.1 & 38.5 & 66.1 & 67.3 \\
fruit\_jelly & $+\,$SPARC & 96.8 & 96.4 & 95.4 & 98.9 & 99.8 & 92.9 & 93.1 & 66.1 & 67.3 \\
rice & base & 54.7 & 39.3 & 49.9 & 81.5 & 82.2 & 47.6 & 66.5 & 74.5 & 74.1 \\
rice & $+\,$SPARC & 80.4 & 78.6 & 52.3 & 93.7 & 94.6 & 69.7 & 88.3 & 74.5 & 74.1 \\
sheet\_metal & base & 77.5 & 79.0 & 50.5 & 83.8 & 78.9 & 71.4 & 41.4 & 62.7 & 56.6 \\
sheet\_metal & $+\,$SPARC & 99.7 & 97.8 & 76.8 & 97.6 & 97.9 & 89.0 & 97.0 & 62.7 & 56.6 \\
vial & base & 91.4 & 69.7 & 68.2 & 90.5 & 88.4 & 71.7 & 69.4 & 27.9 & 57.6 \\
vial & $+\,$SPARC & 93.9 & 89.3 & 87.3 & 96.4 & 99.7 & 90.4 & 99.2 & 27.9 & 57.6 \\
wallplugs & base & 50.8 & 46.4 & 41.2 & 43.1 & 49.1 & 59.6 & 42.4 & 39.7 & 38.7 \\
wallplugs & $+\,$SPARC & 73.6 & 72.2 & 46.9 & 48.0 & 57.1 & 71.9 & 67.2 & 39.7 & 38.7 \\
walnuts & base & 84.5 & 88.2 & 67.3 & 79.5 & 69.9 & 81.8 & 65.1 & 33.6 & 42.2 \\
walnuts & $+\,$SPARC & 90.6 & 94.8 & 81.9 & 87.3 & 91.4 & 89.0 & 90.1 & 33.6 & 42.2 \\
\bottomrule
\end{tabular}
}%
\par\vspace{0.4em}
\caption{\textbf{Per-category AU-PRO$_{0.3}$ on MVTec AD~2.} Results use $k{=}8$.}
\label{tab:supp-percat-mvtecad2-aupro}
\centering
\setlength{\tabcolsep}{2.2pt}
\resizebox{\textwidth}{!}{%
\begin{tabular}{llrrrrrrrrr}
\toprule
Category & Variant & PatchCore & PaDiM & SPADE & AnomalyDINO & SubspaceAD & WinCLIP & MuSc & AnomalyCLIP$_{\mathrm{V}}$ & AnomalyCLIP$_{\mathrm{M}}$ \\
\midrule
can & base & 22.2 & 23.1 & 29.5 & 24.8 & 26.7 & 43.9 & 45.1 & 9.6 & 9.8 \\
can & $+\,$SPARC & 27.4 & 31.5 & 41.3 & 34.7 & 33.4 & 43.9 & 42.0 & 8.5 & 8.3 \\
fabric & base & 22.6 & 27.9 & 13.4 & 50.6 & 62.4 & 33.8 & 33.6 & 34.4 & 26.3 \\
fabric & $+\,$SPARC & 24.6 & 34.3 & 15.5 & 58.5 & 68.0 & 39.0 & 45.0 & 41.4 & 31.2 \\
fruit\_jelly & base & 62.7 & 60.1 & 66.7 & 65.2 & 78.8 & 56.7 & 63.1 & 54.9 & 57.1 \\
fruit\_jelly & $+\,$SPARC & 63.4 & 62.2 & 67.4 & 68.6 & 82.3 & 58.8 & 65.0 & 54.1 & 57.7 \\
rice & base & 38.8 & 49.3 & 37.6 & 72.7 & 74.7 & 32.1 & 54.4 & 51.9 & 44.6 \\
rice & $+\,$SPARC & 45.1 & 52.3 & 44.9 & 71.5 & 77.1 & 35.3 & 63.2 & 54.3 & 48.2 \\
sheet\_metal & base & 22.6 & 21.6 & 19.2 & 69.0 & 39.4 & 19.4 & 25.5 & 31.2 & 29.8 \\
sheet\_metal & $+\,$SPARC & 23.8 & 22.5 & 21.0 & 73.9 & 48.1 & 21.2 & 27.9 & 29.1 & 29.7 \\
vial & base & 85.1 & 63.2 & 88.6 & 89.9 & 94.9 & 85.4 & 89.9 & 67.8 & 64.1 \\
vial & $+\,$SPARC & 87.5 & 84.6 & 89.3 & 93.2 & 96.1 & 86.8 & 92.3 & 70.4 & 64.6 \\
wallplugs & base & 31.9 & 32.4 & 28.4 & 30.9 & 34.1 & 33.6 & 45.1 & 37.8 & 38.9 \\
wallplugs & $+\,$SPARC & 34.9 & 38.6 & 37.3 & 31.6 & 60.2 & 37.6 & 47.0 & 33.2 & 40.2 \\
walnuts & base & 56.5 & 50.7 & 52.6 & 76.3 & 79.5 & 54.9 & 67.6 & 73.3 & 70.1 \\
walnuts & $+\,$SPARC & 59.3 & 54.3 & 55.8 & 77.7 & 89.2 & 54.0 & 71.0 & 69.1 & 70.9 \\
\bottomrule
\end{tabular}
}%
\par\vspace{0.6em}
\centering
\setlength{\tabcolsep}{2.2pt}
\renewcommand{\arraystretch}{0.90}
\par\vspace{0.3em}
\caption{\textbf{Per-category Image AUROC on AeBAD-S.} Results use $k{=}8$.}
\label{tab:supp-percat-aebad-auroc}
\centering
\setlength{\tabcolsep}{2.2pt}
\resizebox{\textwidth}{!}{%
\begin{tabular}{llrrrrrrrrr}
\toprule
Category & Variant & PatchCore & PaDiM & SPADE & AnomalyDINO & SubspaceAD & WinCLIP & MuSc & AnomalyCLIP$_{\mathrm{V}}$ & AnomalyCLIP$_{\mathrm{M}}$ \\
\midrule
background & base & 70.5 & 56.7 & 28.9 & 68.3 & 74.4 & 72.0 & 23.2 & 79.2 & 76.2 \\
background & $+\,$SPARC & 67.8 & 57.5 & 34.9 & 70.9 & 88.5 & 72.6 & 32.8 & 79.2 & 76.2 \\
illumination & base & 64.6 & 68.7 & 41.9 & 75.3 & 53.2 & 73.9 & 35.9 & 76.6 & 71.5 \\
illumination & $+\,$SPARC & 70.3 & 69.6 & 51.1 & 77.4 & 89.7 & 75.8 & 41.2 & 76.6 & 71.5 \\
same & base & 67.4 & 61.0 & 53.3 & 75.9 & 91.3 & 75.8 & 39.5 & 77.3 & 74.5 \\
same & $+\,$SPARC & 69.4 & 61.8 & 56.1 & 78.7 & 93.1 & 76.3 & 47.5 & 77.3 & 74.5 \\
view & base & 53.8 & 60.0 & 29.2 & 67.2 & 75.6 & 81.6 & 24.3 & 82.4 & 78.7 \\
view & $+\,$SPARC & 57.8 & 61.2 & 32.2 & 66.8 & 81.0 & 80.8 & 34.4 & 82.4 & 78.7 \\
\bottomrule
\end{tabular}
}%
\par\vspace{0.4em}
\caption{\textbf{Per-category AU-PRO$_{0.3}$ on AeBAD-S.} Results use $k{=}8$.}
\label{tab:supp-percat-aebad-aupro}
\centering
\setlength{\tabcolsep}{2.2pt}
\resizebox{\textwidth}{!}{%
\begin{tabular}{llrrrrrrrrr}
\toprule
Category & Variant & PatchCore & PaDiM & SPADE & AnomalyDINO & SubspaceAD & WinCLIP & MuSc & AnomalyCLIP$_{\mathrm{V}}$ & AnomalyCLIP$_{\mathrm{M}}$ \\
\midrule
background & base & 89.1 & 83.5 & 90.7 & 85.8 & 89.5 & 82.7 & 76.5 & 85.8 & 89.6 \\
background & $+\,$SPARC & 89.0 & 83.7 & 90.9 & 88.1 & 92.6 & 82.3 & 77.6 & 80.0 & 87.6 \\
illumination & base & 83.9 & 81.3 & 87.4 & 75.6 & 83.7 & 82.1 & 79.5 & 89.0 & 90.6 \\
illumination & $+\,$SPARC & 86.1 & 81.8 & 87.7 & 82.8 & 93.1 & 82.7 & 84.4 & 88.9 & 90.4 \\
same & base & 85.2 & 83.7 & 90.5 & 90.5 & 93.8 & 83.8 & 80.0 & 89.5 & 90.5 \\
same & $+\,$SPARC & 86.6 & 84.1 & 91.0 & 91.9 & 94.1 & 83.7 & 80.1 & 88.3 & 89.6 \\
view & base & 82.4 & 79.6 & 87.2 & 83.8 & 90.0 & 81.8 & 71.4 & 88.1 & 89.0 \\
view & $+\,$SPARC & 83.6 & 80.8 & 88.0 & 84.1 & 90.9 & 81.6 & 68.3 & 87.9 & 89.1 \\
\bottomrule
\end{tabular}
}%
\end{table*}

\begin{table*}[!t]
\caption{\textbf{RAD per-category Image AUROC ($k{=}8$).}}
\label{tab:supp-percat-rad-auroc}
\centering
\setlength{\tabcolsep}{2.2pt}
\renewcommand{\arraystretch}{0.85}
\resizebox{\textwidth}{!}{%
\begin{tabular}{llrrrrrrrrr}
\toprule
Category & Variant & PatchCore & PaDiM & SPADE & AnomalyDINO & SubspaceAD & WinCLIP & MuSc & AnomalyCLIP$_{\mathrm{V}}$ & AnomalyCLIP$_{\mathrm{M}}$ \\
\midrule
bolt & base & 79.1 & 92.5 & 13.0 & 95.1 & 95.6 & 75.2 & 75.8 & 76.1 & 87.5 \\
bolt & $+\,$SPARC & 90.8 & 93.9 & 15.8 & 94.9 & 96.0 & 81.8 & 67.5 & 76.1 & 87.5 \\
ribbon & base & 96.6 & 99.2 & 77.0 & 98.2 & 96.6 & 85.1 & 81.0 & 98.7 & 92.7 \\
ribbon & $+\,$SPARC & 97.4 & 99.1 & 84.8 & 97.9 & 96.7 & 86.1 & 78.7 & 98.7 & 92.7 \\
sponge & base & 95.5 & 99.2 & 7.2 & 97.2 & 96.0 & 68.1 & 55.4 & 43.8 & 36.6 \\
sponge & $+\,$SPARC & 97.6 & 99.2 & 13.1 & 97.0 & 96.0 & 72.9 & 59.1 & 43.8 & 36.6 \\
tape & base & 97.3 & 99.2 & 18.5 & 98.2 & 96.8 & 93.4 & 49.9 & 99.9 & 97.7 \\
tape & $+\,$SPARC & 98.8 & 99.2 & 25.2 & 98.0 & 96.9 & 92.9 & 43.9 & 99.9 & 97.7 \\
\bottomrule
\end{tabular}
}%

\vspace{0.3em}
\caption{\textbf{RAD per-category AU-PRO$_{0.3}$ ($k{=}8$).}}
\label{tab:supp-percat-rad-aupro}
\centering
\setlength{\tabcolsep}{2.2pt}
\renewcommand{\arraystretch}{0.85}
\resizebox{\textwidth}{!}{%
\begin{tabular}{llrrrrrrrrr}
\toprule
Category & Variant & PatchCore & PaDiM & SPADE & AnomalyDINO & SubspaceAD & WinCLIP & MuSc & AnomalyCLIP$_{\mathrm{V}}$ & AnomalyCLIP$_{\mathrm{M}}$ \\
\midrule
bolt & base & 83.4 & 82.5 & 95.6 & 88.3 & 97.5 & 60.9 & 96.4 & 77.2 & 79.7 \\
bolt & $+\,$SPARC & 87.9 & 84.1 & 94.8 & 92.1 & 97.6 & 64.2 & 95.8 & 76.7 & 80.3 \\
ribbon & base & 47.6 & 69.6 & 75.3 & 92.4 & 95.6 & 52.9 & 91.3 & 88.5 & 79.7 \\
ribbon & $+\,$SPARC & 56.4 & 69.9 & 75.5 & 92.4 & 95.7 & 51.9 & 88.8 & 88.0 & 80.2 \\
sponge & base & 69.1 & 60.1 & 56.1 & 79.3 & 85.6 & 24.6 & 79.4 & 43.1 & 15.3 \\
sponge & $+\,$SPARC & 73.6 & 61.2 & 54.6 & 79.5 & 87.6 & 25.7 & 70.5 & 44.5 & 16.5 \\
tape & base & 92.9 & 89.6 & 92.6 & 97.5 & 97.2 & 84.1 & 95.5 & 97.6 & 93.4 \\
tape & $+\,$SPARC & 93.7 & 90.0 & 92.2 & 97.9 & 97.4 & 85.3 & 94.4 & 97.2 & 93.3 \\
\bottomrule
\end{tabular}
}%
\par\vspace{0.6em}
\caption{\textbf{Per-category Image AUROC on VisA.} Results use $k{=}8$.}
\label{tab:supp-percat-visa-auroc}
\centering
\setlength{\tabcolsep}{2.2pt}
\renewcommand{\arraystretch}{0.88}
\resizebox{\textwidth}{!}{%
\begin{tabular}{llrrrrrrrrr}
\toprule
Category & Variant & PatchCore & PaDiM & SPADE & AnomalyDINO & SubspaceAD & WinCLIP & MuSc & AnomalyCLIP$_{\mathrm{V}}$ & AnomalyCLIP$_{\mathrm{M}}$ \\
\midrule
candle & base & 98.6 & 92.9 & 89.4 & 96.0 & 96.5 & 90.8 & 92.4 & 81.2 & 88.0 \\
candle & $+\,$SPARC & 98.9 & 92.9 & 89.0 & 97.1 & 96.7 & 89.8 & 94.1 & 81.2 & 88.0 \\
capsules & base & 76.9 & 70.8 & 70.6 & 97.0 & 98.4 & 64.2 & 84.6 & 82.7 & 91.3 \\
capsules & $+\,$SPARC & 76.0 & 71.0 & 70.3 & 97.2 & 98.4 & 64.9 & 79.3 & 82.7 & 91.3 \\
cashew & base & 96.8 & 95.2 & 99.2 & 97.8 & 98.8 & 83.6 & 92.4 & 76.3 & 96.0 \\
cashew & $+\,$SPARC & 96.6 & 95.6 & 98.6 & 94.7 & 98.8 & 84.2 & 87.2 & 76.3 & 96.0 \\
chewinggum & base & 98.5 & 98.2 & 92.9 & 98.6 & 99.5 & 86.9 & 97.4 & 97.2 & 97.9 \\
chewinggum & $+\,$SPARC & 98.5 & 98.4 & 92.5 & 98.1 & 99.4 & 88.5 & 96.0 & 97.2 & 97.9 \\
fryum & base & 93.0 & 87.8 & 89.8 & 98.1 & 99.1 & 73.5 & 97.3 & 93.0 & 95.9 \\
fryum & $+\,$SPARC & 94.1 & 88.4 & 89.7 & 97.1 & 99.1 & 75.8 & 94.7 & 93.0 & 95.9 \\
macaroni1 & base & 97.5 & 87.9 & 59.1 & 88.1 & 98.6 & 59.4 & 87.8 & 86.5 & 88.5 \\
macaroni1 & $+\,$SPARC & 94.2 & 87.9 & 65.6 & 85.0 & 98.6 & 60.0 & 85.8 & 86.5 & 88.5 \\
macaroni2 & base & 76.2 & 77.1 & 64.5 & 83.4 & 97.2 & 53.5 & 59.0 & 73.0 & 78.5 \\
macaroni2 & $+\,$SPARC & 77.5 & 77.7 & 62.4 & 84.2 & 97.3 & 58.5 & 59.7 & 73.0 & 78.5 \\
pcb1 & base & 98.2 & 92.4 & 90.5 & 96.4 & 97.6 & 85.2 & 87.2 & 85.3 & 91.4 \\
pcb1 & $+\,$SPARC & 97.8 & 92.6 & 91.1 & 92.2 & 97.5 & 86.2 & 86.6 & 85.3 & 91.4 \\
pcb2 & base & 97.6 & 87.8 & 92.6 & 94.1 & 96.0 & 61.7 & 87.2 & 62.4 & 69.3 \\
pcb2 & $+\,$SPARC & 96.7 & 88.2 & 92.5 & 87.6 & 96.1 & 62.9 & 84.7 & 62.4 & 69.3 \\
pcb3 & base & 98.3 & 85.5 & 79.3 & 97.4 & 96.7 & 60.4 & 94.7 & 61.6 & 76.6 \\
pcb3 & $+\,$SPARC & 97.4 & 85.9 & 79.6 & 94.8 & 97.1 & 61.2 & 92.5 & 61.6 & 76.6 \\
pcb4 & base & 99.6 & 99.0 & 91.3 & 98.5 & 98.7 & 92.4 & 92.0 & 94.1 & 97.8 \\
pcb4 & $+\,$SPARC & 99.9 & 99.0 & 91.3 & 98.0 & 98.9 & 92.9 & 95.1 & 94.1 & 97.8 \\
pipe\_fryum & base & 99.9 & 97.4 & 85.1 & 98.4 & 99.8 & 93.7 & 96.8 & 92.4 & 95.2 \\
pipe\_fryum & $+\,$SPARC & 99.6 & 97.7 & 87.4 & 98.1 & 99.8 & 94.0 & 96.0 & 92.4 & 95.2 \\
\bottomrule
\end{tabular}
}%
\par\vspace{0.4em}
\caption{\textbf{Per-category AU-PRO$_{0.3}$ on VisA.} Results use $k{=}8$.}
\label{tab:supp-percat-visa-aupro}
\centering
\setlength{\tabcolsep}{2.2pt}
\resizebox{\textwidth}{!}{%
\begin{tabular}{llrrrrrrrrr}
\toprule
Category & Variant & PatchCore & PaDiM & SPADE & AnomalyDINO & SubspaceAD & WinCLIP & MuSc & AnomalyCLIP$_{\mathrm{V}}$ & AnomalyCLIP$_{\mathrm{M}}$ \\
\midrule
candle & base & 97.1 & 96.4 & 97.0 & 96.9 & 97.8 & 94.1 & 97.3 & 93.7 & 94.8 \\
candle & $+\,$SPARC & 97.0 & 96.4 & 96.4 & 96.7 & 97.6 & 94.0 & 97.4 & 93.6 & 94.6 \\
capsules & base & 72.5 & 57.5 & 76.1 & 96.4 & 98.4 & 51.2 & 85.5 & 78.6 & 83.4 \\
capsules & $+\,$SPARC & 69.8 & 57.8 & 74.3 & 95.8 & 98.2 & 50.9 & 84.2 & 78.4 & 81.3 \\
cashew & base & 93.6 & 85.2 & 93.7 & 94.9 & 98.2 & 82.6 & 93.7 & 90.4 & 94.7 \\
cashew & $+\,$SPARC & 91.5 & 85.0 & 91.8 & 95.3 & 98.2 & 82.4 & 89.2 & 91.6 & 94.5 \\
chewinggum & base & 82.4 & 82.5 & 73.4 & 87.1 & 92.4 & 67.7 & 81.3 & 80.9 & 84.8 \\
chewinggum & $+\,$SPARC & 84.7 & 82.3 & 72.3 & 87.6 & 92.0 & 70.4 & 78.8 & 81.6 & 85.4 \\
fryum & base & 87.9 & 85.6 & 90.0 & 89.5 & 94.0 & 79.9 & 92.9 & 84.4 & 89.5 \\
fryum & $+\,$SPARC & 86.5 & 85.6 & 87.1 & 91.3 & 94.0 & 79.6 & 92.3 & 86.8 & 89.3 \\
macaroni1 & base & 96.3 & 92.4 & 97.2 & 91.9 & 98.1 & 75.9 & 96.1 & 88.6 & 91.2 \\
macaroni1 & $+\,$SPARC & 96.0 & 92.5 & 95.3 & 92.6 & 98.0 & 75.8 & 93.7 & 90.9 & 92.4 \\
macaroni2 & base & 66.3 & 63.8 & 68.1 & 96.3 & 98.9 & 50.9 & 85.3 & 88.5 & 88.2 \\
macaroni2 & $+\,$SPARC & 61.8 & 63.8 & 66.7 & 93.3 & 98.9 & 54.9 & 72.1 & 82.4 & 87.1 \\
pcb1 & base & 95.0 & 87.5 & 95.7 & 94.3 & 96.4 & 60.4 & 93.1 & 76.8 & 77.9 \\
pcb1 & $+\,$SPARC & 93.2 & 87.7 & 94.1 & 94.3 & 96.1 & 63.9 & 91.2 & 74.7 & 76.5 \\
pcb2 & base & 90.1 & 85.4 & 90.8 & 89.2 & 93.8 & 66.2 & 87.1 & 70.7 & 76.6 \\
pcb2 & $+\,$SPARC & 90.3 & 85.6 & 90.5 & 88.8 & 93.8 & 67.1 & 85.9 & 71.4 & 79.6 \\
pcb3 & base & 87.6 & 82.6 & 88.7 & 86.0 & 94.2 & 77.4 & 91.9 & 66.8 & 76.7 \\
pcb3 & $+\,$SPARC & 86.8 & 82.9 & 86.9 & 83.7 & 93.8 & 77.8 & 90.2 & 66.7 & 78.0 \\
pcb4 & base & 89.5 & 83.0 & 88.1 & 89.3 & 92.9 & 82.0 & 91.7 & 87.1 & 89.3 \\
pcb4 & $+\,$SPARC & 85.5 & 82.9 & 81.2 & 87.7 & 92.4 & 81.8 & 89.3 & 86.0 & 89.3 \\
pipe\_fryum & base & 95.5 & 92.3 & 96.2 & 96.7 & 97.7 & 94.5 & 97.3 & 94.0 & 96.9 \\
pipe\_fryum & $+\,$SPARC & 94.6 & 92.3 & 95.4 & 96.7 & 97.7 & 94.1 & 97.4 & 93.7 & 97.0 \\
\bottomrule
\end{tabular}
}%
\par\vspace{0.8em}
\end{table*}

\FloatBarrier
\begin{sidewaysfigure*}[p]
\centering
\includegraphics{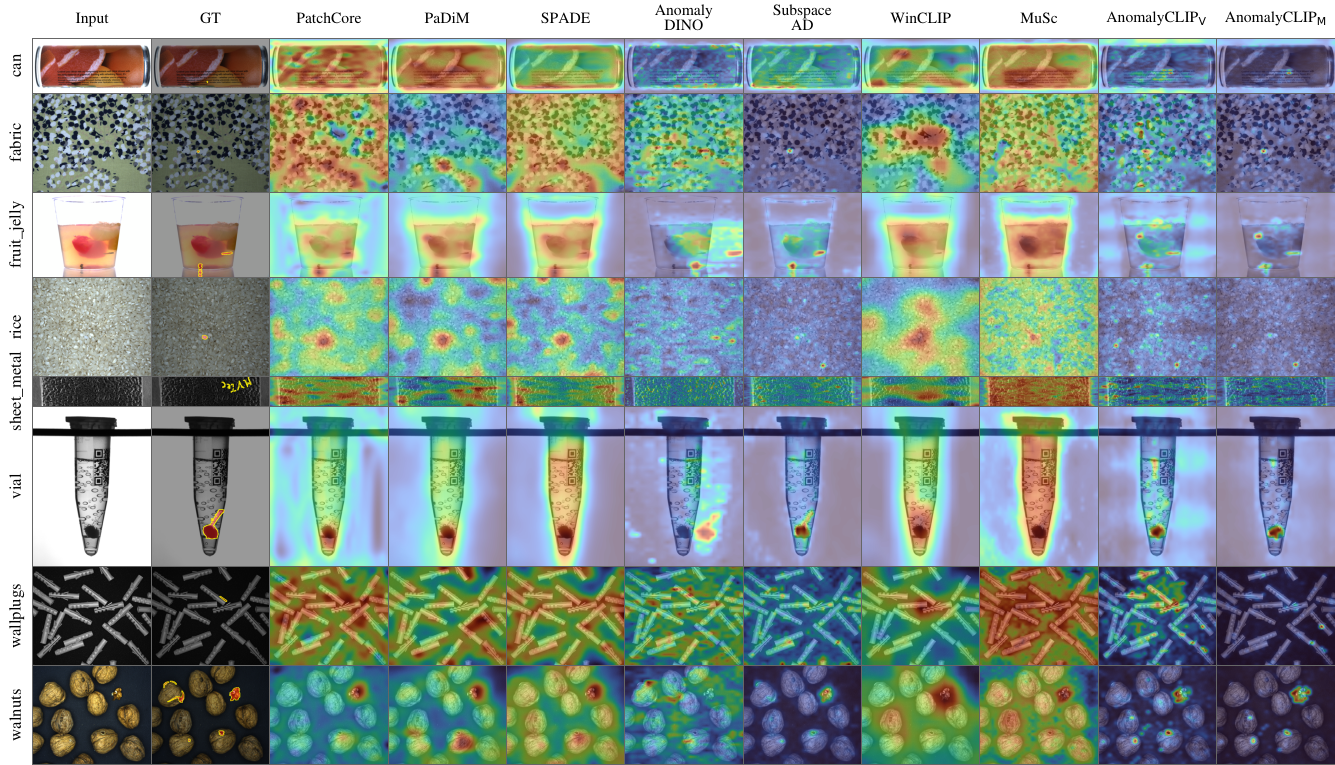}%
\caption{SPARC-corrected qualitative localization on MVTec AD~2. All anomaly maps use SPARC-corrected features; the nine detector columns are min-max normalized independently per image and detector and rendered with a shared turbo colormap. Quantitative results are reported in the main tables. For each category, the displayed test image is the one selected by lower median GT-mask area; filenames are sorted in ascending order on ties ($k{=}8$, seed=0).}
\label{fig:qual-mvtec2}
\end{sidewaysfigure*}

\begin{sidewaysfigure*}[p]\centering
\includegraphics{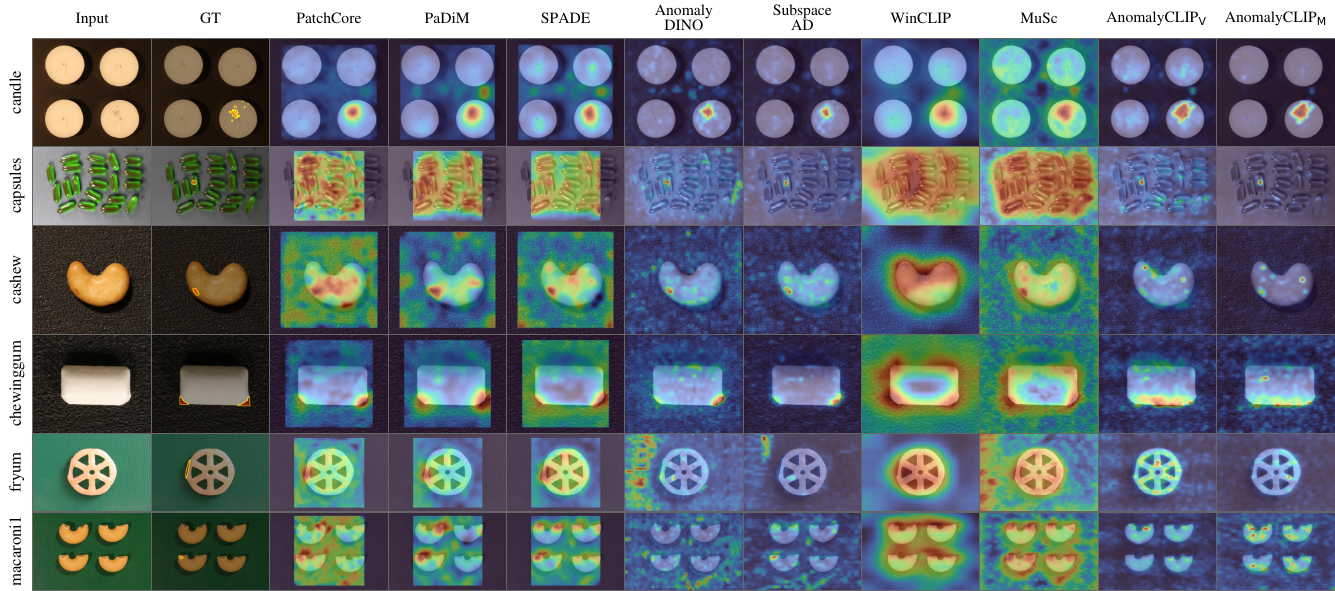}%
\caption{SPARC-corrected qualitative localization on VisA. All anomaly maps use SPARC-corrected features; the nine detector columns are min-max normalized independently per image and detector and rendered with a shared turbo colormap. Quantitative results are reported in the main tables. For each category, the displayed test image is the one selected by lower median GT-mask area; filenames are sorted in ascending order on ties ($k{=}8$, seed=0). (page 1 of 2)}
\label{fig:qual-visa-p1}
\end{sidewaysfigure*}

\begin{sidewaysfigure*}[p]\centering
\includegraphics{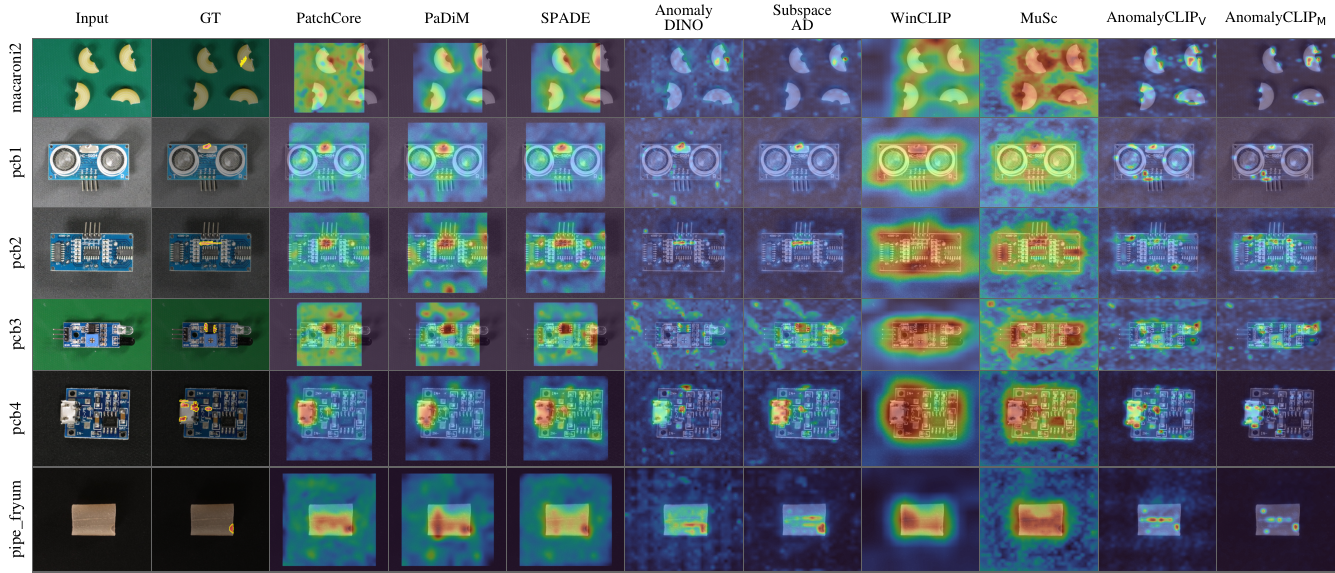}%
\caption{SPARC-corrected qualitative localization on VisA. (page 2 of 2)}
\label{fig:qual-visa-p2}
\end{sidewaysfigure*}

\end{document}